%% file: arxiv.tex
\documentclass{article}

\usepackage[preprint]{corl_2026} % Use this for the initial submission.

\newcommand{\methodname}{Interaction-weighted Resampling}
\newcommand{\methodshort}{IWR}

\usepackage[utf8]{inputenc}
\usepackage[T1]{fontenc}

\usepackage{url}
\usepackage{booktabs}
\usepackage{amsmath,amssymb,amsfonts,amsthm,mathtools}
\usepackage{enumitem}
\usepackage{bbm}
\usepackage{titletoc}

\usepackage{minitoc}
\usepackage{minitoc}
\usepackage{xcolor}         % colors
\usepackage{csquotes}
\usepackage{graphicx}
\usepackage{changepage}
\usepackage{algorithm}
\usepackage{algorithmic}
\usepackage{multirow}
\usepackage{wrapfig}
\usepackage[table]{xcolor}
\usepackage{xurl}
\usepackage{subcaption}
\definecolor{bluegray}{rgb}{0.4, 0.6, 0.8}
\definecolor{lightcarminepink}{rgb}{0.9, 0.4, 0.38}
\definecolor{darkred}{rgb}{0.55, 0.0, 0.0}

\hypersetup{
  colorlinks=true,
  linkcolor=darkred,
  citecolor=brown,
  urlcolor=darkred
}

\definecolor{funcBlue}{RGB}{43,145,175} 
\definecolor{commentTeal}{RGB}{110,154,155}
\setcitestyle{numbers, square, sort&compress, open={[}, close={]}}
\bibpunct{[}{]}{,}{n}{}{,}

\usepackage{microtype}
\usepackage{xcolor}
\input{math.tex}

\title{Learning Object Manipulation from Scratch via Contrastive Interaction}

\author{
Tongle Shen$^{1}$,
Caleb Chuck$^{2,\dagger}$,
Fan Feng$^{1,\dagger}$,
Biwei Huang$^{1,\dagger}$\\
$^{1}$UC San Diego \quad
$^{2}$UT Austin\\
\texttt{\{t7shen, f2feng, bih007\}@ucsd.edu; calebc@cs.utexas.edu}\\
{\small $^{\dagger}$Equal advising}
}
\begin{document}
\maketitle

%===============================================================================

\begin{abstract}
    \input{sections/abstract}
\end{abstract}



%===============================================================================

\input{sections/intro}
\input{sections/relatedwork}
\input{sections/preliminaries}

\input{sections/loc2mani}
\input{sections/method}
\input{sections/experiments}

\input{sections/conclusion}

% \input{sections/theory_bundle}

%===============================================================================
%===============================================================================

%\clearpage
% The acknowledgments are automatically included only in the final and preprint versions of the paper.
\acknowledgments{We thank Jinzhou Tang, Minghao Fu, and Xinyue Wang for fruitful and insightful discussions.}

%===============================================================================
\clearpage
\bibliography{main.bib}  % .bib

\newpage
\setcounter{page}{1}
\appendix
\begin{center}
\hrule height .5pt
\vspace{4mm}
{\Large\textbf{{{Appendix of} \textit{Learning Object Manipulation from Scratch via Contrastive Interaction}}}}
\vspace{4mm}
\hrule height .5pt
\end{center}
\startcontents[sections]
{%
\hypersetup{linkcolor=black}%
\printcontents[sections]{l}{1}{\setcounter{tocdepth}{3}}%
}
\vspace{4mm}
\hrule height .5pt
\vspace{8mm}

\setcounter{figure}{0}
\renewcommand{\thefigure}{A\arabic{figure}}
\setcounter{table}{0}
\renewcommand{\thetable}{A\arabic{table}}
\setcounter{section}{0}
    \renewcommand{\thesection}{\Alph{section}}
    \renewcommand{\theequation}{A\arabic{equation}}
    \setcounter{equation}{0}
    
    \setcounter{figure}{0}
    \renewcommand{\thefigure}{A\arabic{figure}}
    \setcounter{table}{0}
    \renewcommand{\thetable}{A\arabic{table}}
    \renewcommand{\thesection}{\Alph{section}}
    \renewcommand{\theequation}{A\arabic{equation}}
    \setcounter{equation}{0}
\input{appendix/theory}
\input{appendix/related_work}
\input{appendix/iwr_details}
\input{appendix/training_setup}

\input{appendix/simulation_setup}
\input{appendix/real_setup}
\input{appendix/energy_visualization}

\input{appendix/ablation}

\end{document}

%% file: math.tex
\newtheorem{definition}{Definition}
\newtheorem{assumption}{Assumption}
\newtheorem{proposition}{Proposition}

\newtheorem{lemma}{Lemma}

%% file: sections/abstract.tex
Contrastive Reinforcement Learning (CRL) has seen recent success in a wide variety of goal-conditioned robotics tasks by learning structured representations of the dynamics. However, despite its success in locomotion and similar domains, CRL often struggles in interaction-rich manipulation. We argue that a key source of this difficulty is object-centric \textit{interaction}, such as contact or grasping, that induces distinct changes in the underlying dynamic modes. In this work, we formulate manipulation dynamics as a piecewise-smooth Markov process and show that interaction-induced mode changes create piecewise nonlinear reachability structures that are difficult for standard CRL energy functions to represent and plan over. Based on this analysis, we introduce \methodname{} (\methodshort{}). \methodshort{} performs interaction-aware resampling around phases before, during, and after interactions, encouraging the learned representations to better capture the dynamic complexity around interactions. Across interaction-centric environments, including 2D dynamic control, robotic manipulation, and robot air hockey, \methodshort{} improves both sample efficiency and overall performance over prior CRL methods, with around $19.8\%$ average improvement in simulation. Finally, using a sim-to-real pipeline with policies trained by \methodshort{}, we demonstrate the first real-world goal-conditioned robot air hockey agent capable of hitting goals, improving success rate from $25\%$ to $60\%$.

\centering{Project Page: \href{IWR-arxiv.github.io}{IWR-arxiv.github.io}}

%% file: sections/intro.tex
%\vspace{-1.5em}
\section{Introduction}
%\vspace{-1em}
Object manipulation remains a core challenge for robot learning, but the factors that make it difficult to learn generalist manipulation policies remain unclear. Most existing approaches rely heavily on either expert demonstrations~\citep{vecerik2017leveraging, zitkovich2023rt, kim2024openvla, chi2025visuomotor, o2024open} or carefully designed reward functions~\citep{ma2024eureka, tang2025deep}. While these paradigms have impressive progress, they also introduce strong supervision requirements:  demonstrations can be expensive to collect and may cover only a narrow range of behaviors, while task-specific rewards often require substantial engineering and do not easily scale to open-ended multitask settings. These limitations motivate unsupervised approaches that can acquire reusable manipulation skills and representations directly from interaction, without assuming demonstrations or manually specified rewards~\citep{eysenbach2018diversity, touati2021learning, laskin2022unsupervised, park2024metra, agarwal2025unified}. Among them, goal-conditioned Reinforcement Learning (GCRL) offers a promising direction for learning a family of policies to handle the inherent multitask nature of manipulation~\citep{liu2022goal}. Successful GCRL requires learning strong representations to capture the shared properties between different control behaviors. While the strategies for learning these representations remain an open challenge, recent work in Contrastive Reinforcement Learning (CRL) offers a promising direction forward~\citep{eysenbach2022contrastive}. CRL has shown impressive capabilities to learn complex control from scratch, and even limited success in manipulation tasks~\citep{liu2025single}. However, despite the clear fit of GCRL algorithms for manipulation, their performance remains limited.

%In this work, we ask a central question: \textit{when and why does CRL fail for manipulation}, and \textit{what is needed to make it work}? We identify \textit{interactions} as a core property of manipulation that makes effective policy representation difficult. Specifically, these relational events between entities induce distinct changes in dynamic modes. 
%These induced mode changes make the CRL representation no longer well approximated by the linear-Gaussian structure observed in locomotion tasks~\citep{eysenbach2024inference}, thereby making planning over this representation space substantially more difficult. To formalize this challenge, we view manipulation as a mode-switching dynamical system, where the active mode changes across phases before, during, and after an interaction. These interaction-dependent modes induce piecewise nonlinear dynamics that are substantially harder to represent than the smoother structures often observed in locomotion tasks. This provides a concrete mechanism for why CRL representations can become poorly shaped in manipulation: they may smooth over critical interaction boundaries and fail to capture the mode changes that determine future reachability. We empirically observe this failure mode and show how it leads to poor policy performance, as illustrated in Figure~\ref{fig:intro_energy}.

In this work, we ask a central question: \textit{when and why does CRL fail for manipulation}, and \textit{what is needed to make it work}? We identify \textit{interactions} as a core property of manipulation that makes effective policy representation difficult. Specifically, interactions are relational events between entities that induce distinct changes in dynamic modes. Since CRL relies on the learned representation to define goal-reaching energies for planning, these interaction-induced mode changes can substantially distort the geometry over which policies are learned and planned.

To formalize this challenge, we view manipulation as a mode-switching dynamical system, where the active mode changes across phases before, during, and after an interaction. These interaction-dependent modes induce piecewise nonlinear dynamics and reachability structures that are substantially harder to represent than the smoother structures often observed in locomotion tasks, which can be well approximated by linear-Gaussian dynamics~\citep{eysenbach2024inference}. This provides a concrete mechanism for why CRL representations can become poorly shaped in manipulation: they may smooth over critical interaction boundaries and fail to capture the mode changes that determine future reachability. We empirically observe this failure mode and show how it leads to poor policy performance in Fig.~\ref{fig:example}.
\begin{wrapfigure}{r}{0.48\textwidth}
    \centering
    \includegraphics[width=0.47\textwidth]{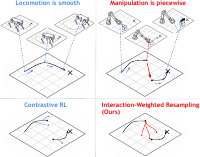}
    \caption{\textbf{Motivation}. In locomotion-like domains, future reachability often follows a smooth temporal structure, making standard CRL effective.  In manipulation, object interactions induce piecewise mode changes, causing standard CRL to miss critical interaction transitions.  IWR  emphasizes these interaction-relevant points to better learn the goal-reaching energy.
}    
\label{fig:example}    
\vspace{-15pt}
\end{wrapfigure}
Guided by our analysis of how interactions induce piecewise nonlinear structure in manipulation, we propose a principled yet simple scheme to improve CRL in these settings. Building on insights from interaction modeling~\citep{chuck2023granger, hwang2024fine, chuck2025null, lei2026spartan} and combinatorial resampling~\citep{kim2024enhancing, ziarko2026contrastive}, we introduce \methodname{} (\methodshort{}), a contrastive resampling principle that reshapes the training distribution toward interaction-induced mode transitions. Instead of treating all transitions uniformly, \methodshort{} constructs informative state-goal comparisons around interaction-relevant points, encouraging the representation to preserve the mode boundaries that govern future reachability. 
Through evaluation on robotic manipulation domains with complex object interactions, including Meta-World, simulated dynamic maze environments, and robot air hockey, we show that \methodshort{} provides a principled way to address the core challenges of CRL by improving both its training distribution and representation capacity. \methodshort{} achieves significantly better sample efficiency and overall performance than prior GCRL methods. Notably, we demonstrate the first goal-conditioned robot air hockey agent capable of hitting multiple goals from a dynamically moving puck.

%Our approach is based on two key insights: \textbf{1)} Combining prior work related to interaction~\citep{} and combinatorial resampling~\citep{}, we introduce \methodname{}, which directs the representations to emphasize interaction-relevant points. \textbf{2)} By recognizing the challenge of learning a representation that can handle multiple modes, we introduce \archname{,} a network augmentation that allows the model to capture greater complexity in how the future distribution can be reached from the current distribution, including nonlinear relationships. We demonstrate that these two modifications are principled solutions to the core challenge by modeling the distribution and representation capacity of CRL. Combining \methodname{} and \archname{}, we are able to demonstrate significantly improved sample efficiency and overall performance compared to prior GCRL work in interaction-centric environments, including: dynamic maze domains, robotic manipulation, and robot air hockey, demonstrating the first instance of a goal-conditioned robot air hockey agent that can hit multiple goals from a dynamically moving puck. 

%In summary, our contributions are:
%\begin{enumerate}
%    \item Formalizing the core challenge of representation learning for robot manipulation with interactions
%    \item Introducing \methodname{} and \archname{} to address these challenges, and empirically validating their effectiveness.
%    \item Providing proofs for the theoretical properties of the interaction weighting
%\end{enumerate}

Our contributions are: \textbf{\textit{(i)}} a principled analysis of why interaction-rich manipulation poses a fundamental challenge for CRL
% , showing how interaction-induced mode changes lead to piecewise nonlinear reachability structures that are difficult to represent and plan over; 
\textbf{\textit{(ii)}} Introducing \methodname{} (\methodshort{}), inspired by this principled analysis, which reshapes the effective contrastive training distribution toward interaction-relevant transitions.
% to capture multi-modal, piecewise nonlinear reachability; 
\textbf{\textit{(iii)}} Empirical validation of \methodshort{}
% , showing how it improves the learning of interaction-relevant representations, and empirically validate our method 
across dynamic maze domains, robotic manipulation, and substantial improvement in robot air hockey.

%% file: sections/relatedwork.tex
%\vspace{-1em}
\section{Related Work}
%\vspace{-1em}
Our work lies at the intersection of CRL, unsupervised goal-conditioned control, and robot manipulation with interactions. An extended discussion on these works is given in Appendix~\ref{app: rl}.

\textbf{Contrastive Reinforcement Learning (CRL)} applies contrastive learning~\citep{oord2018representation} to goal-conditioned reinforcement learning (GCRL), through both classification~\citep{eysenbach2021clearning} and regression  objectives~\citep{eysenbach2022contrastive}. Recent work has extended CRL toward metric learning~\citep{zheng2025multistep,myers2026offline}, planning~\citep{eysenbach2024inference}, language alignment~\citep{myers2023goal,myers2026temporal}, and combinatorial reasoning and search~\citep{wang2026temporal,ziarko2026contrastive}. \methodname{} is most closely related to~\citet{ziarko2026contrastive}, which also uses resampling. However, \methodshort{} targets resampling using interaction structure to focus on difficult-to-model dynamics, whereas~\citet{ziarko2026contrastive} uses different contexts for in-trajectory sampling.
We also build on recent successes of CRL in simulated manipulation~\citep{wang20261000}, including emergent exploration~\citep{liu2025single,bastankhah2026demystifying}, which enables limited but sample-efficient manipulation performance compared to prior GCRL methods~\citep{puterman1990markov,kaelbling1993learning,liu2022goal,andrychowicz2017hindsight,bai2021addressing}. Beyond CRL, \methodshort{} draws on broader contrastive learning ideas, including the use of privileged information such as labels or structured metadata~\citep{feng2022adaptive,denize2023similarity,khosla2020supervised,hoffmann2022ranking}, and debiasing strategies that carefully select negatives~\citep{chuang2020debiased,huynh2022boosting,dwibedi2021little}. Our view of the tension between non-smooth dynamics and smooth contrastive representations is partly inspired by~\citet{betser2026infonce}, which observes that learned representations tend toward a thin-shell Gaussian. Furthermore, the use of factorization and entity interactions in contrastive learning remains limited, which is the focus of \methodshort{}.

\textbf{Manipulation Learning with Interactions} 
Robot manipulation covers tasks such as grasping, moving, and throwing objects~\citep{han2023survey}. Interactions~\citep{chuck2024automated} have been used in manipulation learning through controllability~\citep{seitzer2021causal}, data augmentation~\citep{pitis2020counterfactual,pitis2022mocoda}, hierarchical RL~\citep{chuck2020hypothesis,chuck2023granger,wang2024skild}, skill learning~\citep{hu2024disentangled,rodriguez2025pixels,hosseini2026susd}, causal modeling~\citep{lee2023scale,biswas2024gaze}, and exploration~\citep{wang2023elden}. In contrast, our work formalizes manipulation difficulty as interaction-induced mode-changing dynamics and introduces \methodshort{} as an interaction-aware algorithm for training CRL. \methodshort{} improves sample efficiency and, to our knowledge, provides the first CRL application to robot air hockey~\citep{chuck2024air}.

%% file: sections/preliminaries.tex
%\vspace{-1em}
\section{Background}
%\vspace{-1em}
In this section, we set up the basic modeling framework. Specifically, we model manipulation as a Markov Decision Process with factored structures that describe interactions, and then learn goal-conditioned policies using CRL within the GCRL framework.

\textbf{Factorized Markov Decision Process} We consider a Markov decision process (MDP) $\mathcal{M}=(\mathcal{S},\mathcal{A},P,r,\gamma)$ with a factored state space~\citep{guestrin2003efficient}. Specifically, the state is decomposed into $n$ components, $s = (s^{1}, s^{2}, \ldots, s^{n}) \in \mathcal{S},
    \quad 
    \mathcal{S} = \mathcal{S}^{1} \times \mathcal{S}^{2} \times \cdots \times \mathcal{S}^{n}$.
Each factor $s^{i}$ represents the state of an entity, such as the robot, an object, and its properties. This factorization gives a natural way to describe manipulation domains with multiple interacting entities.

\textbf{Goal-Conditioned Reinforcement Learning}
In Goal-Conditioned Reinforcement Learning (GCRL)~\citep{kaelbling1993learning}, an agent learns the policy that can solve a family of tasks specified by different goals. Let $\mathcal{S}$ denote the state space, $\mathcal{A}$ the action space, and $\mathcal{G}$ the goal space. At the beginning of each episode, a goal $g \in \mathcal{G}$ is sampled from a goal distribution $p(g)$. The agent then acts according to a goal-conditioned policy $\pi(a \mid s, g)$, which maps the current state and desired goal to a distribution over actions. 
Given transition dynamics $p(s_{t+1}\mid s_t,a_t)$ and a goal-conditioned reward function $r:\mathcal{S}\times\mathcal{A}\times\mathcal{G}\rightarrow \mathbb{R}$, the objective is to maximize the expected discounted return over both trajectories and goals: $J(\pi)
    =
    \mathbb{E}_{g \sim p(g),\; \tau \sim p_\pi(\tau \mid g)}
    \left[
    \sum_{t=0}^{T} \gamma^t r(s_t,a_t,g)
    \right]$,
where $\tau=(s_0,a_0,\ldots,s_T)$ is a trajectory induced by $\pi(\cdot\mid s,g)$ and $\gamma \in [0,1)$ is the discount factor. 
%In many manipulation settings, the reward is sparse and only indicates whether the achieved state is close to the desired goal, \textit{e.g.},
%$r(s_t,a_t,g) = \mathbbm{1}\{d(\phi(s_{t+1}), g) \leq \epsilon\}$,
%where $\phi(s)$ extracts the achieved goal from the state and $d(\cdot,\cdot)$ measures goal distance.

\textbf{Contrastive Reinforcement Learning} We consider goals as future states or achieved goal states. For a state-action pair $(s,a)$, we define the discounted future occupancy under policy $\pi$ as
\begin{equation*}
    \rho^\pi(g\mid s,a)
    =
    (1-\gamma)\sum_{k=1}^{\infty}\gamma^{k-1}
    p^\pi(s_{t+k}=g\mid s_t=s,a_t=a).
\end{equation*}
This essentially measures how likely the policy is to reach goal $g$ in the future after taking action $a$ at state $s$. Thus, learning a goal-conditioned policy can be viewed as learning a representation of state-goal reachability.

Given this goal-conditioned setting, contrastive reinforcement learning (CRL)~\citep{eysenbach2022contrastive} learns a goal-reaching energy that measures how likely a future goal $g$ is reachable from a state-action pair $(s,a)$. CRL can be viewed as estimating the density ratio between this conditional future occupancy $\rho^\pi(g\mid s,a)$ and a replay-marginal future distribution $\bar{\rho}^{\pi}_{\mathcal{B}}(g)$:
\begin{equation}
    E^\star(s,a,g)
    =
    \log \rho^\pi(g\mid s,a)
    -
    \log \bar{\rho}^{\pi}_{\mathcal{B}}(g).
\end{equation}
In practice, CRL parameterizes this energy with an inner product between a state-action encoder and a goal encoder~\citep{eysenbach2022contrastive,liu2025singlegoal}. Let
$\phi:\mathcal{S}\times\mathcal{A}\rightarrow\mathbb{R}^d$
and
$\psi:\mathcal{S}\rightarrow\mathbb{R}^d$.
The learned energy is
\begin{equation}
    E_{\phi,\psi}(s,a,g)=\phi(s,a)^\top\psi(g).
\label{eq:inner_energy}
\end{equation}
For each state-action pair $(s,a)\sim d_{\mathcal{B}}(s,a)$, CRL samples one positive future
$g^{(1)}\sim\rho^\pi(\cdot\mid s,a)$
and $N-1$ negative futures
$g^{(2)},\ldots,g^{(N)}\sim\bar{\rho}^{\pi}_{\mathcal{B}}(\cdot)$.
The critic is trained with an InfoNCE-style objective, together with a LogSumExp regularizer:
\begin{equation}
  \max_{\phi,\psi}\;
  \mathbb{E}
  \Bigg[
  \log
  \frac{
  \exp(\phi(s,a)^\top\psi(g^{(1)}))
  }{
  \sum_{j=1}^N \exp(\phi(s,a)^\top\psi(g^{(j)}))
  }
  -
  \beta_{\rm lse}
  \left(
  \log\sum_{j=1}^N \exp(\phi(s,a)^\top\psi(g^{(j)}))
  \right)^2
  \Bigg],
  \label{eq:sgcrl-critic-objective}
\end{equation}
where the expectation is over
$(s,a)\sim d_{\mathcal{B}}(s,a)$,
$g^{(1)}\sim\rho^\pi(\cdot\mid s,a)$, and
$g^{(2)},\ldots,g^{(N)}\sim\bar{\rho}^{\pi}_{\mathcal{B}}(\cdot)$.
The learned energy is then used as an implicit goal-reaching reward for actor learning.

%The learned energy is then used as an implicit goal-reaching reward for actor learning. Thus, the success of CRL depends on whether the inner-product energy can faithfully represent the reachability structure of the environment. This representation is effective when reachability varies smoothly, but can become insufficient in manipulation domains where interactions induce sharp mode changes and piecewise nonlinear future distributions.

% Placeholder: definition and a short summary

% Placeholder: An elegant way to define object-object relational process.
% In this paper we only consider state input, and fmdp could be easily
% achieved by object-centric model, or 3d point cloud, anything has a rough
% distance measure

%% file: sections/loc2mani.tex
%\vspace{-1em}
\section{Analysis: Why does CRL Struggle with Interactions in Manipulation?}
%\vspace{-1em}
\label{sec: analysis}
%The success of CRL depends on whether the inner-product energy $E_{\phi,\psi}(s,a,g)=\phi(s,a)^\top\psi(g)$ can faithfully represent the future-density ratio $E^\star(s,a,g)
%    =
%    \log \rho^\pi(g\mid s,a)
%    -
%    \log \bar{\rho}^{\pi}_{\mathcal{B}}(g)$.
The success of CRL depends on whether the inner-product energy $E_{\phi,\psi}(s,a,g)$ can faithfully represent the future-density ratio $E^\star(s,a,g)$.
This motivates our analysis: what class of future distributions is naturally represented by the inner-product CRL energy function, and when does this representation become insufficient? Building on the Gaussian idealization behind temporal interpolation in CRL~\citep{eysenbach2024inference}, we first revisit why contrastive representations are well-suited to smooth, single-mode dynamics, and then show why this structure breaks down in manipulation, where interactions induce latent mode changes and piecewise nonlinear future distributions.
%\vspace{-1em}
\subsection{CRL as a Linear-Gaussian Temporal Model}
\label{subsec:crl-gaussian-view}
%\vspace{-1em}
We begin from the Gaussian idealization of temporal contrastive representations. \citet{eysenbach2024inference} show that CRL induces a simple linear-Gaussian model of discounted future occupancy in representation space. In particular, when the marginal representation distribution is approximately isotropic Gaussian and the contrastive density ratio is represented by a quadratic score, the conditional future representation admits a Gaussian form,
\begin{equation*}
    p(\psi_{t+n}\mid \psi_t)
    =
    \mathcal N(\mu_n(\psi_t),\Sigma_n).
\end{equation*}
Naturally, for a single-mode MDP, the temporal mean can be parameterized by a single linear operator $A$, yielding
\begin{equation}
    \mu_n(\psi_t) = A^n\psi_t,
    \qquad
    p(\psi_{t+n}\mid \psi_t)
    =
    \mathcal N(A^n\psi_t,\Sigma_n).
\end{equation}
All temporal structure is generated by powers of the same operator $A$. The geometrically sampled future states used by CRL provide dense supervision for estimating this temporal operator across different horizons, which explains the strength of CRL in navigation and locomotion domains. Essentially, CRL can be understood as learning a smooth Gaussian-like geometry for future reachability.

% This explains the empirical strength of CRL in navigation and locomotion domains~\citep{eysenbach2022contrastive}. When the action affects the same state variables that define the future goal, and the resulting reachability structure is organized by a single smooth temporal mode -- an inner-product CRL critic can learn a representation in which temporal interpolation and planning are well behaved. In this sense, CRL can be understood as learning a smooth Gaussian-like geometry for future reachability.
%\vspace{-0.8em}
\subsection{Manipulation Induces Mode-Conditioned Gaussian Futures}
%\vspace{-0.8em}
\label{subsec:mode-conditioned-gaussian}
However, the single-operator Gaussian view becomes insufficient in manipulation. Unlike locomotion domains, where the action persistently affects the same state variable used as the future goal, manipulation often separates the actuated factor from the target factor. Let the state be factored as $s_t=(u_t,y_t)$,
where $u_t$ denotes the actuated factor, such as the robot end-effector or paddle, and $y_t$ denotes the target factor, such as an object or puck, and the state-action anchor is $x_t=(u_t,y_t,a_t)$.

In the average case, actions do not always affect the target factor. Instead, their effects are activated only through interaction events. We represent this using an interaction indicator $\omega_t=m(x_t)\in\{0,1\}$,
where $\omega_t=0$ denotes \textit{passive} target dynamics and $\omega_t=1$ denotes an interaction mode. The target dynamics can then be written as a piecewise transition:
\begin{equation}
y_{t+1}
=
\begin{cases}
f_0(y_t), & \omega_t=0,\\
f_1(u_t,y_t,a_t), & \omega_t=1.
\end{cases}
\label{eq:piecewise-target-dynamics}
\end{equation}
In the passive mode, the target dynamics are without direct dependence on the robot action; in the interaction mode, the target dynamics depend on the actuated factor and action.

Applying the Gaussian temporal view mode-wise yields a different representation-level future model for each mode. Let \(\psi_t=\psi(y_t)\) be the target-factor representation. In the \textit{passive mode}, the future representation can be approximated by a target-only temporal operator, $\psi_{t+1}=A_0\psi_t+\epsilon_t$.
In the \textit{interaction mode}, however, the target update includes an additional displacement induced by the actuated factor and action:
    $\psi_{t+1}
    =
    A_1\psi_t + b_t + \epsilon_t$,
    $b_t = B(u_t,y_t,a_t)$.

Combining the two cases, manipulation induces a switched affine Gaussian model:
\begin{equation}
    \psi_{t+1}
    =
    A_{\omega_t}\psi_t
    +
    \omega_t b_t
    +
    \epsilon_t,
    \qquad
    \epsilon_t\sim\mathcal N(0,\Sigma_{\omega_t}).
    \label{eq:switched-affine-gaussian}
\end{equation}
This indicates the \textit{structural mismatch} between manipulation and the single-mode CRL geometry. Interactions change the active mode of the representation dynamics and introduce an interaction residual $b_t$ that depends on the robot-object relation and action. As a result, the future representation is no longer governed by powers of a single operator $A$, but by a sequence of mode-dependent operators and interaction residuals. An illustration is in Fig.~\ref{fig:theory_illustration}.
% \begin{figure}[h]
%     \vspace{-8pt}   
%     \centering
%     \small
% \includegraphics[width=1.0\linewidth]{figure/theory-illustration.pdf}
%     \caption{\textbf{Interaction-induced mode changes break the single-operator CRL geometry}. In smooth locomotion domains, future representations can often be organized by repeatedly applying a single temporal operator, producing a coherent trajectory toward the goal. In manipulation, interactions such as contact or object transfer switch the active dynamics and introduce residual displacements that depend on the actuated factor and action. Hence, the future distribution becomes piecewise and branching, making standard CRL smoothing over critical interaction transitions.}
%     \label{fig:theory_illustration}
%     \vspace{-.5cm}
% \end{figure}
%\textcolor{red}{An illustration in Appendix xxx}
%\vspace{-0.8em}
\subsection{Two Challenges for CRL in Interaction-Rich Manipulation}
%\vspace{-0.8em}
The mode-conditioned Gaussian reveals two distinct challenges for applying CRL to manipulation. 

\textcolor{darkred}{\textbf{Sampling Challenge}}: Standard CRL learns from endpoint pairs $(x_t,y_f)$, where $y_f$ is sampled as a discounted future. However, the interaction event that changes the target dynamics may occur only at a sparse and latent time $\tau_t$. If the sampled future lies before the interaction, the tuple only observes passive target evolution and provides no direct supervision for the interaction-induced shift. If the sampled future lies far after the interaction, the effect of the interaction is propagated through subsequent passive dynamics, which can attenuate or mix the signal. As a result, ordinary contrastive sampling can be dominated by smooth pre-interaction or post-interaction transitions, while the critical bridge transition that changes the dynamic mode remains underrepresented. 

\textcolor{darkred}{\textbf{Error Propagation}}: Even when an interaction-crossing future is sampled, the interaction information may be observed only after it has propagated through subsequent target dynamics.
This makes the learned energy fragile around contact.
We formalize this issue in the one-reset case, where the target first evolves passively, then experiences one interaction, and then evolves passively again.

\begin{proposition}[Error Propagation]
\label{prop:endpoint-energy-error} Let $e$ denote the local interaction-bridge estimation error, and let $k$ denote the number of passive steps after the interaction. 
In the one-reset case, the endpoint energy error scales as
\begin{equation}
    \sup|\widehat{E}_k - E_k|
    \propto
    \|A_0^k\|
    \|e\|
    +
    \frac{1}{2}
    \|A_0^k\|^2
    \|e\|^2 .
    \label{eq:energy-error-bound-main}
\end{equation}
\end{proposition}
Proposition~\ref{prop:endpoint-energy-error} shows that the endpoint energy error depends on two factors: the local bridge estimation error $\|e\|$ and the post-interaction propagation factor $\|A_0^k\|$. 
In manipulation, after contact the target often evolves passively and cannot be corrected by the robot.
Therefore, an inaccurate interaction shift is carried forward by the passive dynamics, producing fragile energy estimates and inaccurate actor gradients near the interaction.
In tasks with repeated or branched interactions, such as pick-place, this issue becomes more severe because the future distribution becomes a mixture over multiple interaction times and interaction residuals. More detailed theorems, analysis, and the full proof can be found in Appendix~\ref{app:theory}. 

%% file: sections/method.tex
\section{Method}
\begin{figure}[t]
    \vspace{-8pt}   
    \centering
    \small
\includegraphics[width=1.0\linewidth]{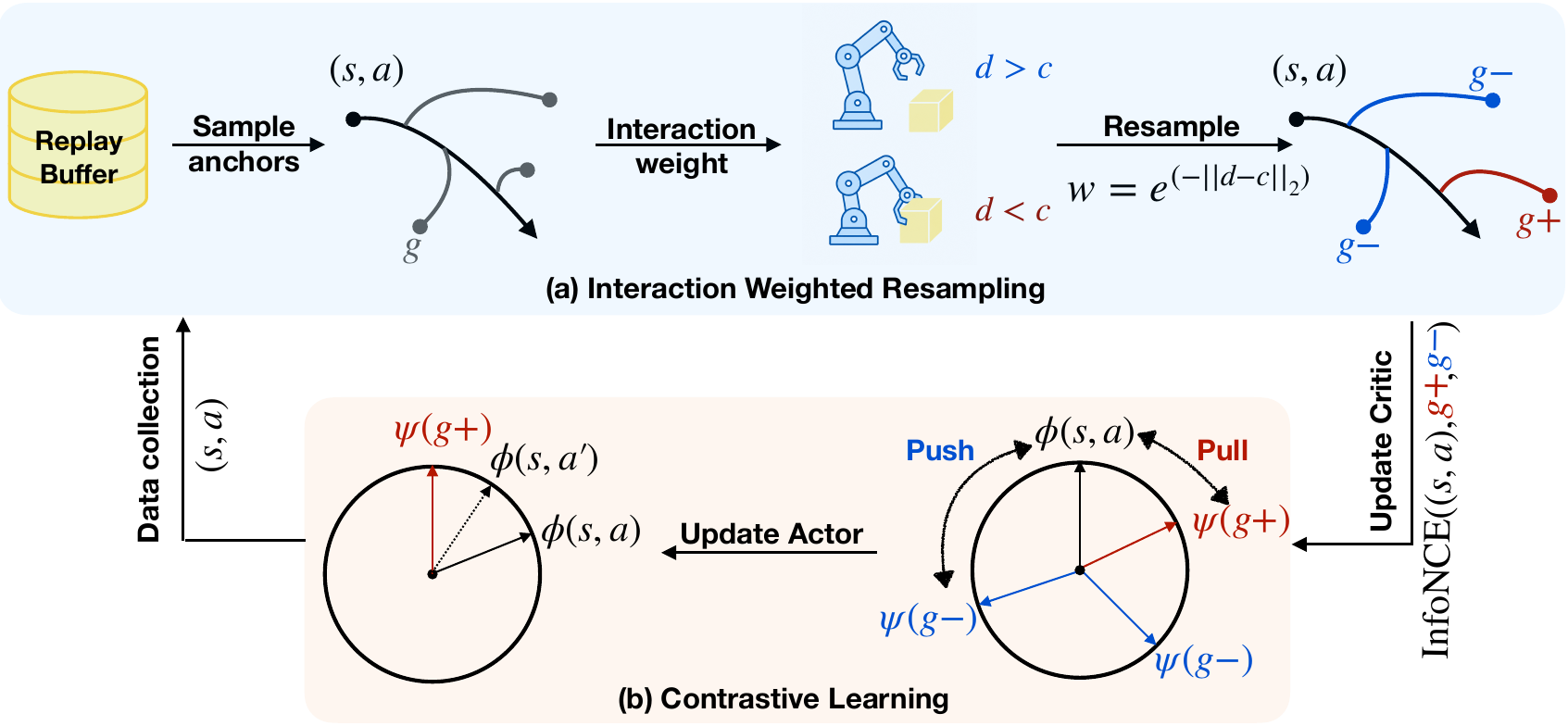}
    \caption{\textbf{Overview of the learning framework}. (a) IWR samples anchor state-action pairs from the replay buffer and reweights candidate future goals by their proximity to the interaction threshold, increasing the chance of selecting interaction-relevant positives. 
(b) The resampled positives and replay-marginal negatives are used in the standard CRL critic update. The learned energy then updates the actor, whose new rollouts are added back to the replay buffer.}
    \label{fig:method}
    \vspace{-.5cm}
\end{figure}
%To address these challenges,  we propose {Interaction-Weighted Resampling (IWR)}, aiming to emphasize contrastive-energy updates on the interaction part.
To address the weak identification of interactions under standard CRL sampling, we propose \emph{Interaction-Weighted Resampling} (IWR). The goal of IWR is to shift the effective contrastive training distribution toward interaction-centered futures, while leaving the CRL critic objective unchanged.

For an anchor transition $x_t=(s_t,a_t)$, we consider a set of candidate future goals $g_k=s_{t+k}$ from the same trajectory. Let $u_{p,t+k}$ and $y_{p,t+k}$ denote the positions of the actuated factor and target factor at the candidate future time $t+k$, respectively. We define their distance as
\begin{equation*}
    d_{t+k}=\|u_{p,t+k}-y_{p,t+k}\|_2 .
\end{equation*}
Let $c$ be a loose interaction threshold, encoding prior knowledge of where interaction is likely to occur, e.g., contact distance between the gripper and object. IWR assigns each candidate a future soft interaction weight:
\begin{equation}
    w_k
    =
    \epsilon
    +
    \exp\left(
        -\frac{||d_{t+k}-c||_2}{2\sigma^2}
    \right),
    \label{eq:iwr}
\end{equation}
where $\epsilon>0$ preserves global coverage and $\sigma$ controls the width of the interaction window. 
%For each sampled transition $(u,y,a)$, denote the position of the actuated factor as $u_p$ and the position of the target factor as $y_p$.
%Let the distance be $d=\|u_p-y_p\|_2$, where $\|\cdot \|_2$ denotes $L_2$ distance obtained from state or visual inputs.
%Denote the interaction-threshold heuristic by $c$, a loose constant that encodes prior knowledge about where interactions are most likely to occur.
%With a small noise term $\epsilon$ and the emphasis power $\sigma$, we calculate the interaction-weighting label $w$ as:
%\begin{equation}
%    w = \epsilon + \frac{1}{\sigma}\exp(-\|d-c\|_2)
%\label{eq:iwr}
%\end{equation}
%Figure~\ref{fig:method} shows the whole learning pipeline of our algorithm.
%We first sample an anchor $(s,a)$ and multiple future goals $g$ from the replay buffer.
%Next, IWR calculates the sampling weight of each future goal using Eq.~\ref{eq:iwr}.
Figure~\ref{fig:method} summarizes the learning pipeline. We first sample anchors \((s,a)\) from the replay buffer and collect candidate future goals from the same trajectory. IWR computes the interaction weights in Eq.~\ref{eq:iwr} and resamples positive futures toward likely interaction regions. The resampled positives are then used in the standard CRL critic objective in Eq.~\ref{eq:sgcrl-critic-objective}. The actor is updated by maximizing the learned goal-reaching energy, and the updated policy collects new transitions into the replay buffer, forming a self-improving loop. Then, the method follows the standard Contrastive RL pipeline: it trains the critic with Eq.~\ref{eq:sgcrl-critic-objective} and updates the actor policy by $\pi = \text{argmax}_a E(s,a,g)$.
Finally, the updated policy collects new transitions in the replay buffer. 
Thus, \methodshort{} focuses the sampling distribution of contrastive updates towards mode-switching regions to improve the modeling of this region. We describe the tighter link between \methodshort{} and Proposition~\ref{prop:endpoint-energy-error} in Appendix~\ref{app:iwr_details}.

%% file: sections/experiments.tex
\section{Experiments}
%\vspace{-0.5em}
\subsection{Simulation Experiment}
%\vspace{-0.5em}
%\begin{figure}[t]
%    \centering
%    \begin{minipage}[t]{0.75\textwidth}
%        \centering
%        \vspace{0pt}
%        \includegraphics[width=\linewidth]{figure/task_set_figure.png}
%        \caption{Task sets used in the simulation experiments.}
%        \label{fig:task_set}
%    \end{minipage}%
%    \begin{minipage}[t]{0.25\textwidth}
%        \centering
%        \vspace{0pt}
%        \includegraphics[width=\linewidth]{figure/air_hockey_setup.png}
%        \caption{Real-world robotic air hockey testbed.}
%        \label{fig:real_world_air_hockey_setup}
%    \end{minipage}
%    \vspace{-1em}
%\end{figure}

\paragraph{Experimental Settings}
We evaluate IWR across three groups of interaction-centric environments: 
(i) \textbf{2D Box2D manipulation.} We test IWR with state inputs on five 2D manipulation tasks in Box2D~\citep{catto2026box2d} (Fig.~\ref{fig:task_set}), where the agent controls the blue ball to move the green ball into the red goal region. We report \emph{success ticks}, defined as the number of ticks for which the green ball remains inside the red circle, normalized by 100 ticks, with each episode lasting 200 ticks. 
(ii) \textbf{Robotic manipulation.} We then evaluate IWR on four manipulation tasks: \textsc{Push}, \textsc{Pick-Place}, \textsc{Sweep-Into}, and \textsc{Peg-Insert}, using success rate as the metric. 
(iii) \textbf{Robot air hockey.} Finally, we test IWR in the Air-Hockey simulator~\citep{chuck2024air} using the velocity-observation mode. The task is to hit a falling puck into a small circular goal of radius $r=0.06$, where the paddle starts on the lower board and strikes the puck toward the goal. We further evaluate performance in a consecutive goal-reaching setting. This setting challenges the trained policy to reach as many goals as possible within one minute. Once a goal is reached, a new goal is sampled. We report both the number of successful goal reaches and the survival rate, which indicates whether the puck avoids falling to the bottom within one minute. For sim-to-real transfer, we add observation noise and switch from velocity observations to history-based observations (Sec.~\ref{subsec:real-robot}).
%We first evaluate IWR with state inputs on five 2D manipulation tasks in Box2D \cite{catto2026box2d} (Fig.~\ref{fig:box2d}), where we control the blue ball to move the green ball into the red circle. The metric is success ticks, which indicates how long the green ball stays in the red circle, normalized by 100 ticks, with an episode length of 200 ticks. Then, we evaluate four manipulation tasks: push, pick-place, sweep-into, and peg-insert. The metric is the success rate. Finally, we test the Air-Hockey simulator \cite{chuck2024air} with the velocity-observation mode. We evaluate success rate, and the goal is a small circle with radius $r=0.06$. The puck falls from the top, and the paddle hits the puck on the lower board to hit the goal. For sim-to-real transfer, we add noise and change the observation mode from velocity to history. See Section~\ref{subsec:real-robot} for sim-to-real transfer.

We compare our method against two groups of baselines.
(i) \textbf{Reward-based RL with sparse rewards.}
We evaluate PPO~\citep{schulman2017proximal} and SAC~\citep{haarnoja2018sac}. For SAC, we also consider variants with a hindsight replay buffer~\citep{andrychowicz2017hindsight} and an interaction replay buffer~\citep{chuck2025null}.
(ii) \textbf{Contrastive RL methods.}
We compare against SGCRL~\citep{liu2025singlegoal} and CRTR~\citep{ziarko2026contrastive}. SGCRL and CRTR differ in whether sampling is repeated $R$ times from a single episode. Since our method belongs to the same family as CRTR, we set $R=8$. CRTR can also be viewed as an ablation of our method with interaction-agnostic uniform sampling.
For the CRL family, after the policy first succeeds at goal reaching, we further evaluate consecutive goal reaching: once a goal is reached, the goal position is changed, and the policy continues acting. We report the average consecutive survival rate, defined by whether the puck avoids falling to the bottom, and the goal-reaching count. Each task is repeated over 5 seeds and details are in Appendix~\ref{app:training_setup}.

\begin{figure*}[t]
    \centering
    \begin{subfigure}{0.33\textwidth}
        \centering
        \includegraphics[width=\linewidth]{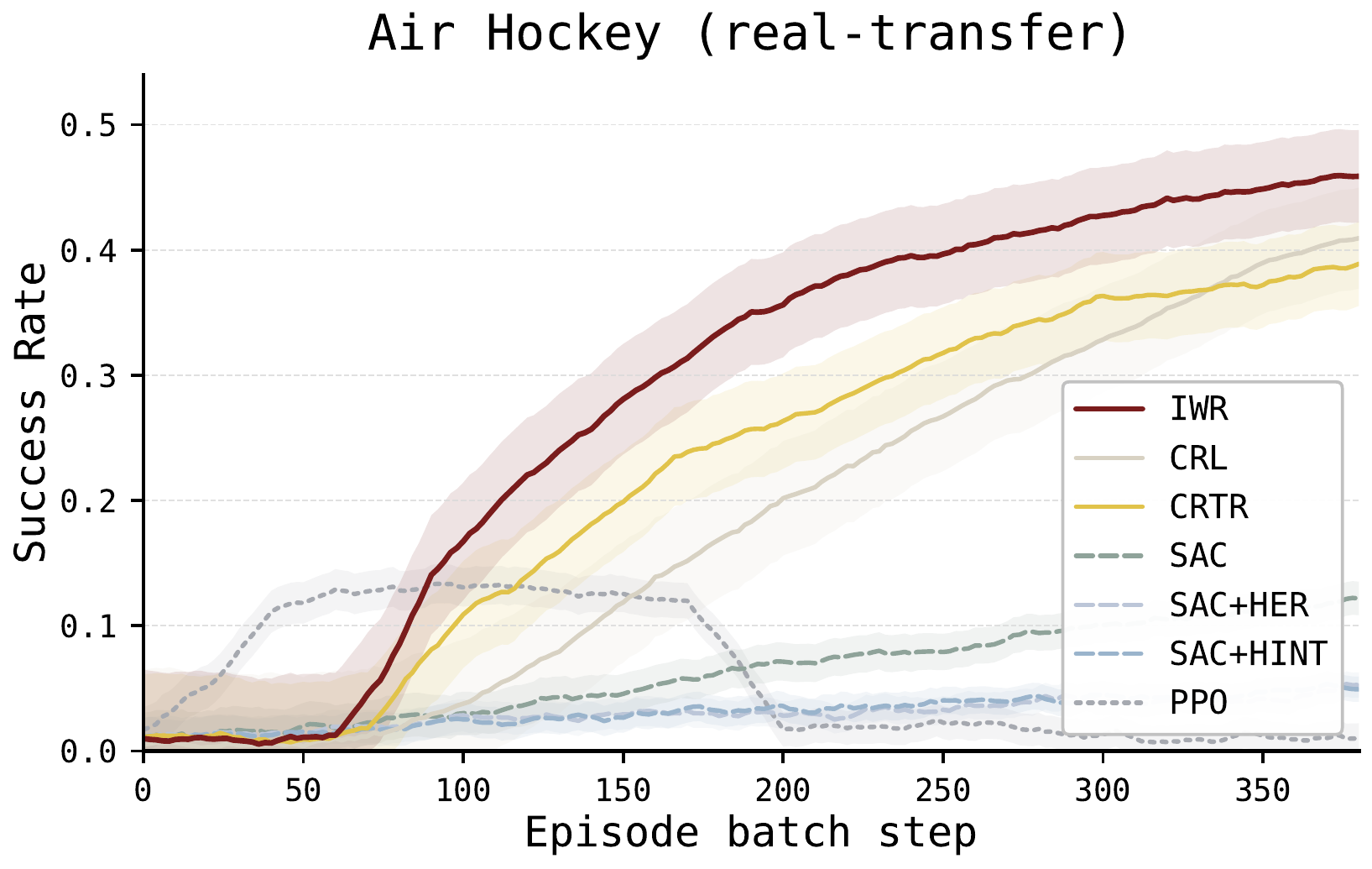}
        \caption{Training curve.}
        \label{fig:air_hockey_real_transfer_curve}
    \end{subfigure}
    \hfill
    \begin{subfigure}{0.41\textwidth}
        \centering
        \includegraphics[width=\linewidth]{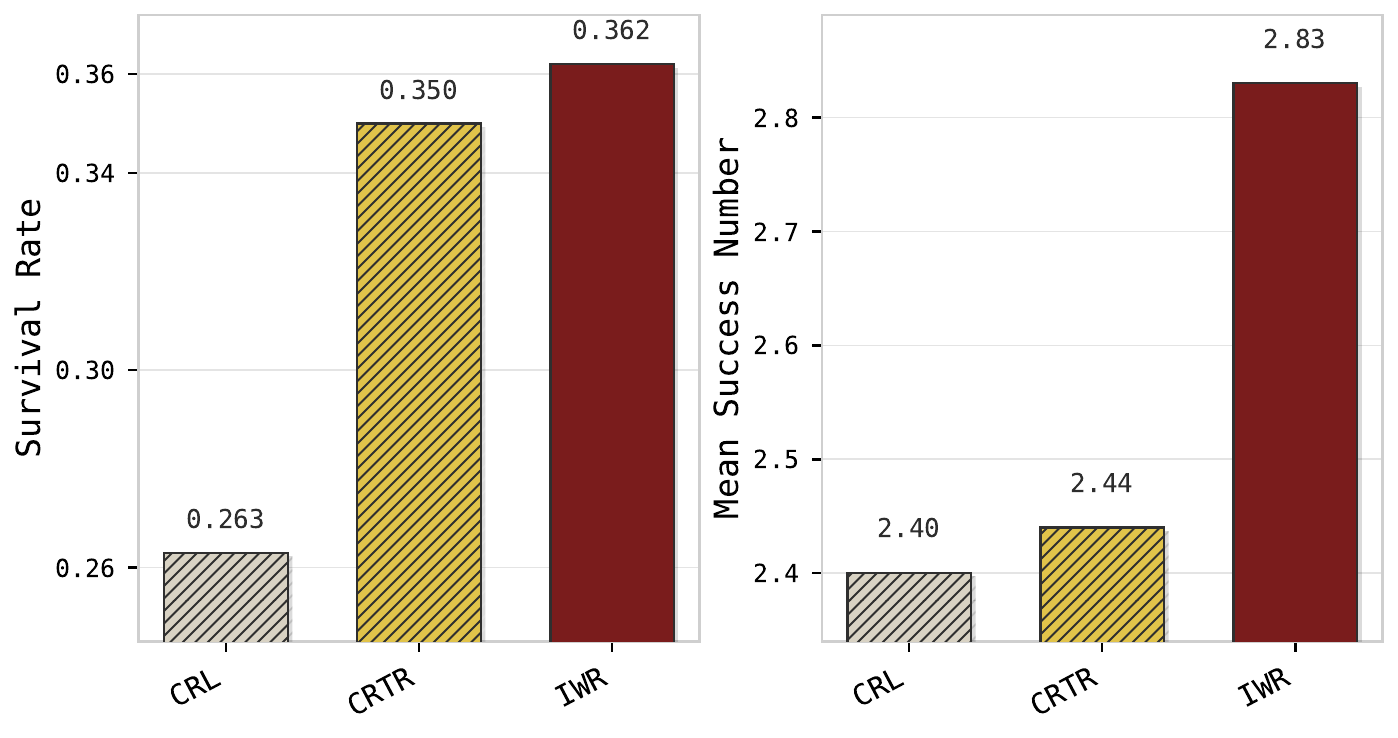}
        \caption{Consecutive evaluation.}
        \label{fig:air_hockey_consecutive_eval}
    \end{subfigure}
    \hfill
    \begin{subfigure}{0.2\textwidth}
        \centering
        \includegraphics[width=\linewidth]{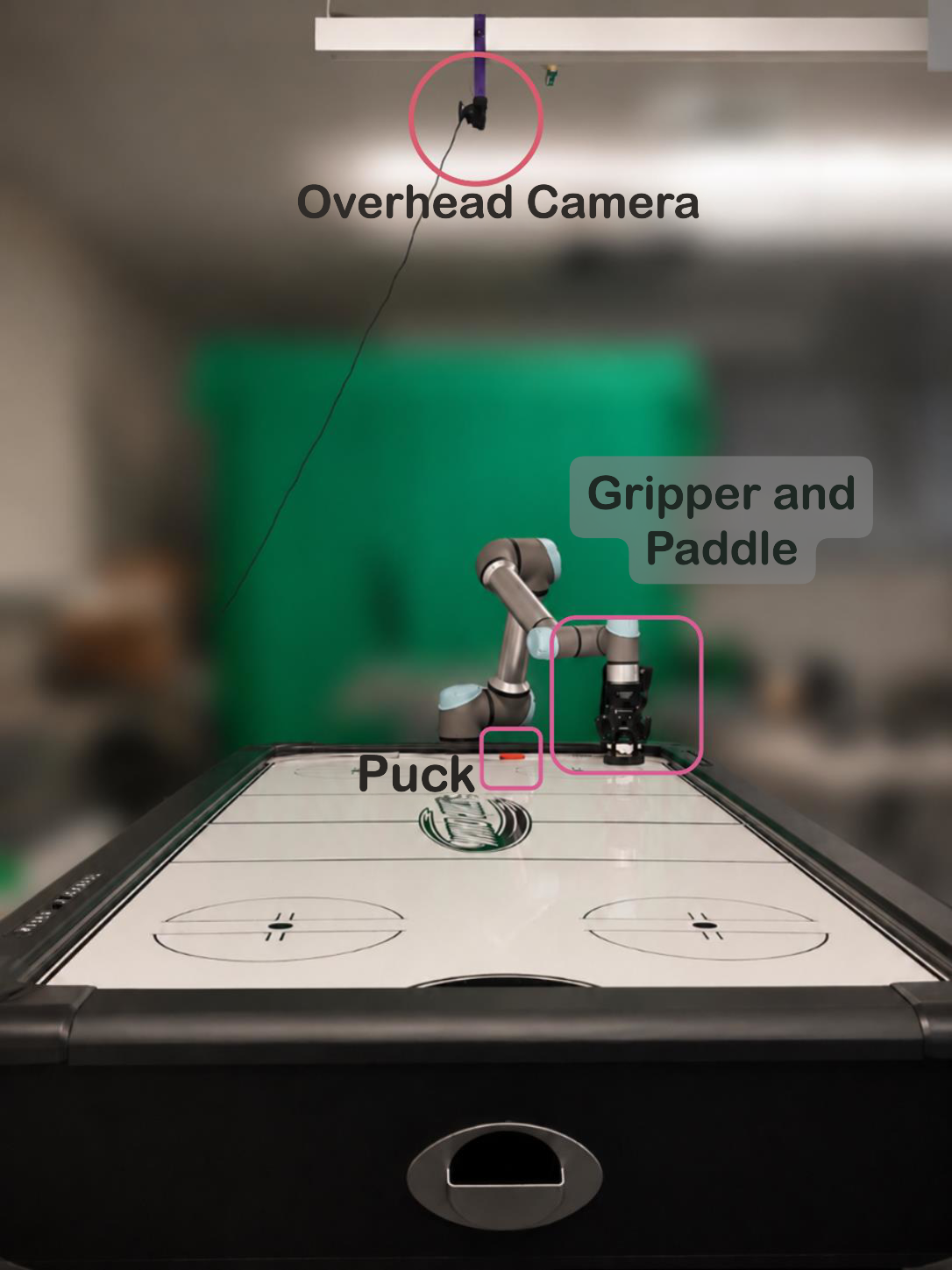}
        \caption{Real robot setup.}
    \label{fig:air_hockey_setup}
    \end{subfigure}
    \caption{
    \textbf{Air hockey real-transfer}.    (a) Training performance in simulation under the real-transfer setting.
    (b) Consecutive evaluation after training in simulator.
    (c) Real robot air hockey setup.
    }
\label{fig:air_hockey_real_transfer}
    \vspace{-0.6cm}
\end{figure*}

\textbf{Results} 
Table~\ref{tab:main1} reports the best average performance of all algorithms. 
The reward-free \textbf{CRL family outperforms the PPO algorithm and SAC family}, even when SAC is augmented with hindsight replay and interaction replay. The reward-based method has no success on MetaWorld and Air Hockey Real Robot. This agrees with previously observed~\citep{eysenbach2021clearning} comparative results.
% efficacy of self-improvement through contrastive representations, which provides better estimation in sparse-reward settings than Bellman-backup-based Q-value estimation.
IWR outperforms prior methods on most tasks, with an average gain of 19.8\%. 
A key pattern is that the gains are \textbf{largest when interaction and goal-reaching signals are sparse}, supporting our hypothesis that IWR amplifies interaction-relevant supervision and learns a better goal-reaching energy landscape.

In \texttt{\textbf{Box2D}}, precise-control tasks require the agent to keep the green ball near the goal center, demanding both exploration and accurate interaction modeling. The hard variants make the objects smaller, further reducing successful interaction signals; here IWR yields larger gains. In the hard-velocity setting, random initial velocities occasionally create early successes, so the relative improvement is smaller but still substantial. This shows that IWR is especially useful when interaction events are rare but necessary for success.
\texttt{\textbf{Metaworld}} shows a similar trend. In \textsc{Push} and \textsc{Sweep-Into}, random behavior can sometimes succeed, and the goal can be reached through multiple strategies, such as pushing or stabbing the object into the goal. In contrast, \textsc{Pick-Place} requires precise long-horizon interaction, since the policy must grasp and carry the object. IWR achieves a much larger gain in this task, suggesting that interaction-weighted sampling extends beyond the simplified one-step setting in Proposition~\ref{prop:endpoint-energy-error} to hybrid dynamics with repeated interaction propagation, e.g., terms of the form $A_1^k$.
For \texttt{\textbf{air hockey}}, the single-goal success-rate gain is modest, but Fig.~\ref{fig:air_hockey_real_transfer_curve} shows great improvements in real-transfer setting, which includes noise and discontinuous observations. This indicates IWR would be consistent when the environment is noisy, and the learned policy and critic is more generalizable with higher sample coverage near interaction.

Fig.~\ref{fig:air_hockey_consecutive_eval} shows the evaluation on consecutive goal-hitting evaluation. This indicates that IWR improves not only isolated successes, but also the robustness and precision of interaction control. The effect is even more pronounced in the real-world setting, as discussed in Sec.~\ref{subsec:real-robot}.

\input{table/table1}
% \input{table/table2}
%\vspace{-0.8em}
\subsection{Real-Robot Experiment}
%\vspace{-0.8em}
\label{subsec:real-robot}
\paragraph{Experimental Settings} In this work, we utilize the setting of \citet{chuck2024robot} as seen in Fig.~\ref{fig:real_world_air_hockey_setup} with a sim-to-real pipeline with domain randomization to transfer policies from the 2D simulator to a UR5e robot on one end of an air hockey table (see Fig.~\ref{fig:paired_sim_real}). The state space is unified between both 2D and 3D through puck detection and coordinate frame transformation. No training is performed on the real robot setup---all policies are transferred sim-to-real, and evaluated on a grid of $4\times5$ goals. Additional robot details, especially the sim-to-real pipeline, can be found in Appendix~\ref{app:real_setup}.
\vspace{-0.8em}
\paragraph{Results} As we see in Table~\ref{tab:main1}, the success rate for \methodshort{} significantly exceeds that of the other methods ($140\%$ more success), even though the gap in simulation (Air Hockey real-transfer) is significantly smaller ($5\%$ increase). We believe this is because when transferring sim-to-real, the qualitative aspects of the policy, that is, how it induces interactions, make a larger difference. Because \methodshort{} emphasizes greater representation capacity for interactions, it makes contact with the puck more consistently, and with greater upward force. This results in more overall strikes, which significantly boosts the performance. Other comparable methods make contact with the puck only once, resulting in unrecoverable failures. Nonetheless, the performance improvement of \methodshort{} is substantial, suggesting that by focusing on interactions, it learns a more robust policy. This is qualitatively visible in the visualizations of the policy performance found in Appendix Figures~\ref{fig:crtr-paired}-\ref{fig:sgcrl-render-paired}.

\begin{figure}
    \centering
    \vspace{-1em}
    \includegraphics[width=0.98\textwidth]{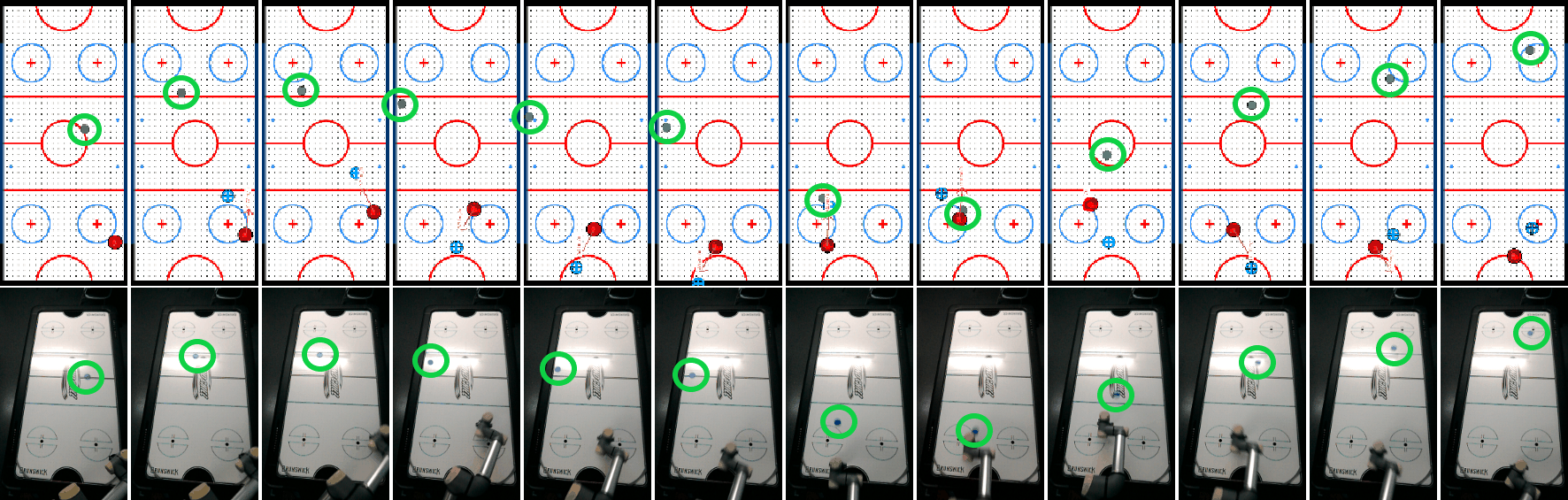}
    \caption{
    2D sim paired with corresponding real frames. Puck emphasized with green circle.
    }
    \label{fig:paired_sim_real}
    \vspace{-1em}
\end{figure}

%% file: table/table1.tex
\definecolor{ourscol}{RGB}{255,238,238}
\definecolor{gainrow}{RGB}{255,248,246}
\definecolor{strongcell}{RGB}{255,220,220}

\begin{table*}[t]
\centering
\small
\setlength{\tabcolsep}{4pt}
\caption{
Success rates over Air Hockey, Box2D, and MetaWorld tasks.
Light red rows highlight interaction-centric tasks where IWR provides larger gains.
}
\label{tab:main1}
\resizebox{\textwidth}{!}{%
\begin{tabular}{lrrrrrrr}
\toprule
Task & PPO & SAC & SAC+HER & SAC+HINT & CRL & CRTR & IWR (Ours) \\
\midrule
Air Hockey (Simulation) & 0.617 & 0.145 & 0.398 & 0.422 & 0.695 & 0.727 & \cellcolor{ourscol}\textbf{0.742 (+2.1\%)} \\
Air Hockey (real-transfer) & 0.160 & 0.215 & 0.129 & 0.125 & 0.477 & 0.465 & \cellcolor{ourscol}\textbf{0.500 (+4.8\%)} \\
\rowcolor{gainrow}
Air Hockey (Real Robot) & 0/20 & 0/20 & 0/20 & 0/20 & 5/20 & 2/20 & \cellcolor{strongcell}\textbf{12/20 (+140.0\%)} \\

\midrule
Box2D (center) & 0.086 & 0.058 & 0.088 & 0.088 & 0.278 & 0.274 & \cellcolor{ourscol}\textbf{0.288 (+3.6\%)} \\
Box2D (goal) & 0.089 & 0.046 & 0.086 & 0.064 & 0.450 & 0.558 & \cellcolor{ourscol}\textbf{0.709 (+27.1\%)} \\
\rowcolor{gainrow}
Box2D (hard) & 0.060 & 0.042 & 0.064 & 0.076 & 0.317 & 0.365 & \cellcolor{strongcell}\textbf{0.565 (+54.8\%)} \\
Box2D (hard velocity) & 0.148 & 0.149 & 0.152 & 0.139 & 0.387 & 0.377 & \cellcolor{ourscol}\textbf{0.436 (+12.7\%)} \\
Box2D (maze) & 0.033 & 0.012 & 0.031 & 0.035 & 0.217 & 0.206 & \cellcolor{ourscol}\textbf{0.223 (+2.8\%)} \\

\midrule
MetaWorld (peg insert) & 0.000 & 0.000 & 0.000 & 0.000 & 0.430 & 0.367 & \cellcolor{ourscol}\textbf{0.438 (+1.9\%)} \\
\rowcolor{gainrow}
MetaWorld (pick place) & 0.000 & 0.000 & 0.004 & 0.000 & 0.266 & 0.305 & \cellcolor{strongcell}\textbf{0.570 (+86.9\%)} \\
MetaWorld (push) & 0.000 & 0.000 & 0.004 & 0.000 & 0.699 & \textbf{0.750} & \cellcolor{ourscol}0.730 \\
MetaWorld (sweep into) & 0.000 & 0.004 & 0.020 & 0.004 & 0.805 & 0.910 & \cellcolor{ourscol}\textbf{0.926 (+1.8\%)} \\

\midrule
Average IWR improvement &  &  &  &  &  &  & \cellcolor{strongcell}\textbf{+19.8\%} \\
\bottomrule
\end{tabular}%
}
\end{table*}

%% file: sections/conclusion.tex
%\vspace{-1em}
\section{Conclusion and Limitations}
%\vspace{-1em}
By formulating manipulation as mode-conditioned dynamics, we illustrate why Contrastive Reinforcement Learning will struggle in manipulation. \methodname{} offers a practical, simple solution that provides an inductive bias towards this complexity, resulting in improved performance across interaction-rich tasks---especially those with rarer or more complex interactions, such as hard dynamic ball interactions, pick-and-place, and real-world air hockey. As robot learning moves towards more complex settings, handling this complexity will become increasingly crucial.

\textbf{Limitations:} Despite the promising success of \methodname{}, it remains an incomplete initial step toward leveraging the insight of mode-conditioned dynamics in robot learning. First, \methodshort{} still represents multimodal manipulation dynamics with a unimodal model---it only changes sampling. Second, CRL is one form of unsupervised RL, with limited representational power, but \methodshort{} can be applied to other unsupervised RL methods. Third, \methodshort{} assumes a factorized state and heuristic interaction indicators---integrating \methodshort{} with an end-to-end perception system is key to future systems. Fourth, CRL incurs a high sample-efficiency cost, preventing it from real-world application---our failures stem from the sim-to-real gap. Thus, future directions involve both algorithmic modifications and better robot systems for transferring core knowledge to the agent.

%% file: appendix/theory.tex
\section{Theory Details for Proposition~\ref{prop:endpoint-energy-error}}
\label{app:theory}
\label{app:one-reset-error}

\newtheorem{corollary}{Corollary}[section]

\subsection{From Factored Dynamics to Local Affine Representation Dynamics}
\label{subsec:local-affine-dynamics}

We first justify the switched affine representation model used in the main paper.
The key point is not that manipulation dynamics are globally linear.
Rather, the claim is that after factorization, the local dependency structure of the target factor changes across passive and interaction modes.
This change forces us to distinguish the passive operator $A_0$ from the interaction operator $A_1$, and to include an interaction-dependent affine term $b_t$.

Let the state be factored as $s_t=(u_t,y_t)$, where $u_t$ is the actuated factor and $y_t$ is the target factor.
We write the target representation as
\[
    \psi_t \triangleq \psi(y_t).
\]
The CRL critic conditions on the anchor through $\phi(s_t,a_t)=\phi(u_t,y_t,a_t)$.
If the action is produced by a deterministic policy, $a_t=\pi(u_t,y_t,g)$, then the same argument can be viewed as conditioning on $(u_t,y_t)$ under the current policy.
Thus, for the theory, we only need to study how the anchor variables induce the next target representation $\psi_{t+1}$.

\paragraph{Dependency contrast.}
In locomotion-like domains, the controlled factor and the goal factor are typically the same state variable.
If this factor is denoted by $z_t$, the local dynamics have the form
\[
    z_{t+1}=f_{\rm loc}(z_t,a_t),
\]
and the action affects the goal-relevant state almost everywhere:
\[
    \left\|
    \frac{\partial f_{\rm loc}(z_t,a_t)}{\partial a_t}
    \right\|
    >0
    \quad
    \text{for a.e. } (z_t,a_t).
\]
Under a fixed policy or fixed anchor action, this gives a single smooth closed-loop transition map.
This is the setting in which a single local temporal operator is a natural approximation for the future representation, consistent with the linear-Gaussian view of temporal contrastive representations~\citep{eysenbach2024inference}.

Manipulation has a different dependency structure.
The target factor is often not directly actuated.
Using the notation from the main paper, the target dynamics are
\[
y_{t+1}
=
\begin{cases}
f_0(y_t), & \omega_t=0,\\
f_1(u_t,y_t,a_t), & \omega_t=1,
\end{cases}
\]
where $\omega_t=0$ denotes passive target dynamics and $\omega_t=1$ denotes an interaction.
Within each smooth mode, the action-to-target derivative is therefore gated by the interaction indicator:
\[
    \frac{\partial y_{t+1}}{\partial a_t}
    =
    \begin{cases}
    0, & \omega_t=0,\\
    \frac{\partial f_1(u_t,y_t,a_t)}{\partial a_t}, & \omega_t=1.
    \end{cases}
\]
Thus, action influence on the target is absent in the passive mode and appears only during interaction.
This is a structural property of manipulation, not a choice of parameterization.

\begin{assumption}[Mode-dependent target controllability]
\label{ass:mode-dependent-controllability}
Let
\[
    H_0(y) \triangleq \psi(f_0(y)),
    \qquad
    H_1(u,y,a) \triangleq \psi(f_1(u,y,a)).
\]
We assume that $H_0$ and $H_1$ are differentiable in the local regions considered.
In the passive mode, the represented target update has no direct action dependence:
\[
    \frac{\partial H_0(y)}{\partial a}=0.
\]
In the interaction mode, there exists a local interaction point
$(\bar u,\bar y,\bar a)$ and an action direction $v$ such that
\[
    D_a H_1(\bar u,\bar y,\bar a)v \neq 0.
\]
\end{assumption}

Assumption~\ref{ass:mode-dependent-controllability} implies that a target-only predictor is appropriate in the passive mode, but insufficient in the interaction mode.
The reason is that two anchors can share the same target state $y_t$ while differing in the actuated factor or action.
During interaction, these two anchors can produce different next target states.
A predictor depending only on $\psi_t$ would assign the same next representation to both anchors, and therefore cannot represent the interaction update.

\begin{lemma}[Need for an interaction-dependent residual]
\label{lem:need-affine-residual}
Suppose Assumption~\ref{ass:mode-dependent-controllability} holds.
In the passive mode, a first-order local predictor of the represented target update can be written as a target-only linear predictor.
In the interaction mode, no target-only linear predictor of the form $A\psi(y)$ can approximate the represented target update to first order in the interaction context.
Therefore, the interaction-mode local predictor requires an additional context-dependent residual, yielding the affine form
\[
    \psi_{t+1}
    =
    A_1\psi_t+b_t .
\]
\end{lemma}

\begin{proof}
We prove the two claims separately.

\textbf{Passive mode.}
In the passive mode,
\[
    \psi_{t+1}
    =
    H_0(y_t)
    =
    \psi(f_0(y_t)).
\]
Since $H_0$ depends only on $y_t$, its first-order local expansion in representation coordinates around $\bar y$ has the form
\[
    H_0(y)
    =
    H_0(\bar y)
    +
    A_0\bigl(\psi(y)-\psi(\bar y)\bigr)
    +
    o\!\left(\|\psi(y)-\psi(\bar y)\|\right).
\]
After centering the local coordinates, or equivalently absorbing the affine intercept into an augmented representation coordinate, this becomes
\[
    \psi_{t+1}
    =
    A_0\psi_t
    +
    o\!\left(\|\psi_t-\psi(\bar y)\|\right).
\]
Thus, a target-only linear predictor is sufficient to first order in the passive mode.

\textbf{Interaction mode.}
Assume for contradiction that a target-only linear predictor is first-order sufficient in the interaction mode.
Then there exists a matrix $A$ such that, for all sufficiently small $h$,
\[
    H_1(\bar u,\bar y,\bar a+hv)
    =
    A\psi(\bar y)
    +
    o(|h|).
\]
Evaluating the same expression at $h=0$ gives
\[
    H_1(\bar u,\bar y,\bar a)
    =
    A\psi(\bar y).
\]
Subtracting the two equations yields
\[
    H_1(\bar u,\bar y,\bar a+hv)
    -
    H_1(\bar u,\bar y,\bar a)
    =
    o(|h|).
\]
Dividing by $h$ and taking $h\rightarrow 0$ gives
\[
    D_aH_1(\bar u,\bar y,\bar a)v
    =
    0.
\]
This contradicts Assumption~\ref{ass:mode-dependent-controllability}, which states that
\[
    D_aH_1(\bar u,\bar y,\bar a)v \neq 0.
\]
Therefore, no target-only linear predictor can be first-order sufficient in the interaction mode.

It remains to show the affine form.
The first-order Taylor expansion of $H_1$ around $(\bar u,\bar y,\bar a)$ is
\[
    H_1(u,y,a)
    =
    H_1(\bar u,\bar y,\bar a)
    +
    A_1\bigl(\psi(y)-\psi(\bar y)\bigr)
    +
    D_uH_1(\bar u,\bar y,\bar a)(u-\bar u)
    +
    D_aH_1(\bar u,\bar y,\bar a)(a-\bar a)
    +
    o(\Delta),
\]
where
\[
    \Delta
    =
    \|\psi(y)-\psi(\bar y)\|
    +
    \|u-\bar u\|
    +
    \|a-\bar a\|.
\]
Rearranging the first-order terms gives
\[
    H_1(u,y,a)
    =
    A_1\psi(y)
    +
    b(u,y,a)
    +
    o(\Delta),
\]
where
\[
    b(u,y,a)
    =
    H_1(\bar u,\bar y,\bar a)
    -
    A_1\psi(\bar y)
    +
    D_uH_1(\bar u,\bar y,\bar a)(u-\bar u)
    +
    D_aH_1(\bar u,\bar y,\bar a)(a-\bar a).
\]
At time $t$, define
\[
    b_t \triangleq b(u_t,y_t,a_t).
\]
Therefore, the interaction-mode local predictor has the affine form
\[
    \psi_{t+1}
    =
    A_1\psi_t+b_t
    +
    o(\Delta).
\]
Dropping the higher-order local error gives the switched affine approximation used in the main analysis.
\end{proof}

% \begin{lemma}[Need for an interaction-dependent residual]
% \label{lem:need-affine-residual}
% Suppose Assumption~\ref{ass:mode-dependent-controllability} holds.
% In the passive mode, a local target predictor can be written as a target-only linear predictor.
% In the interaction mode, a target-only linear predictor is generally insufficient; an additional term depending on the interaction context is required.
% \end{lemma}

% \begin{proof}
% In the passive mode, $y_{t+1}=f_0(y_t)$.
% Therefore, after applying the target representation, the conditional next representation depends locally only on $\psi_t$.
% A local linear regression in representation space can therefore be written as
% \[
%     \psi_{t+1}\approx A_0\psi_t.
% \]

% In the interaction mode, $y_{t+1}=f_1(u_t,y_t,a_t)$.
% By Assumption~\ref{ass:mode-dependent-controllability}, changing $a_t$ while keeping $y_t$ fixed can change the next target state.
% Since $\psi$ is locally non-degenerate on goal-relevant target coordinates, this can also change $\psi_{t+1}$.
% Thus, two anchors with the same $\psi_t$ but different interaction contexts can require different predictions for $\psi_{t+1}$.
% No predictor of the form $A\psi_t$ can distinguish these two cases, because its input is identical.
% Hence, the interaction-mode predictor must include an additional interaction-dependent residual.
% \end{proof}

We now define the local affine approximation used in the proof.
The operator $A_0$ is the local linear predictor for passive target evolution.
The operator $A_1$ is the target-state coefficient in the interaction mode.
The residual $b_t$ is the part of the interaction update induced by the actuated factor and action after accounting for the target-state term.

\begin{assumption}[Mode-wise local affine representation dynamics]
\label{ass:mode-wise-linear}
Along the local trajectory segment analyzed in Appendix~\ref{app:one-reset-error}, the conditional mean dynamics of the target representation satisfy
\[
    \psi_{t+1}
    =
    A_0\psi_t,
    \qquad
    \omega_t=0,
\]
and
\[
    \psi_{t+1}
    =
    A_1\psi_t+b_t,
    \qquad
    \omega_t=1.
\]
Equivalently, the two cases can be written as the switched affine model
\[
    \psi_{t+1}
    =
    A_{\omega_t}\psi_t+\omega_t b_t.
\]
Here $b_t$ is zero outside interaction because it is multiplied by $\omega_t$.
\end{assumption}

This assumption is a local regression statement in representation space.
It does not require the original manipulation dynamics to be globally linear.
Instead, it says that near a trajectory segment, the passive target update is approximated by a target-only linear operator, while the interaction update requires a separate local operator plus an interaction-dependent residual.
This is the representation-level form of the dependency change in the factored dynamics.
The one-reset analysis below studies how an error in this residual is propagated to later endpoint energies.

\subsection{One-Reset Simplification}

We now formalize the one-reset setting used in Proposition~\ref{prop:endpoint-energy-error}.
Let $\tau$ denote the first future interaction offset after anchor time $t$.
The one-reset mode sequence is
\[
    \underbrace{0,\ldots,0}_{\tau\ \text{passive steps}},
    1,
    \underbrace{0,\ldots,0}_{k\ \text{passive steps}},
\]
where $k$ is the number of passive steps after the interaction.
Thus, the endpoint is at time $t+\tau+1+k$.

\begin{definition}[Interaction bridge shift]
\label{def:bridge-shift}
The interaction bridge shift is the difference between the actual interaction update and the passive
update that would have occurred without interaction:
\begin{equation}
    \delta_\tau
    \triangleq
    (A_1-A_0)\psi_{t+\tau}+b_{t+\tau}.
    \label{eq:bridge-shift-appendix}
\end{equation}
Equivalently,
\begin{equation}
    A_1\psi_{t+\tau}+b_{t+\tau}
    =
    A_0\psi_{t+\tau}
    +
    \delta_\tau .
    \label{eq:bridge-as-passive-plus-shift}
\end{equation}
\end{definition}

The bridge shift isolates the effect of the interaction itself.
Before the bridge, the target evolves passively.
At the bridge, the target receives the additional shift $\delta_\tau$.
After the bridge, this shift is carried forward by passive dynamics.

\begin{definition}[Bridge estimation error]
\label{def:bridge-error}
Let $\widehat{\delta}_\tau$ be the bridge shift represented by the learned energy model.
The local interaction-bridge estimation error is
\begin{equation}
    e
    \triangleq
    \widehat{\delta}_\tau-\delta_\tau .
    \label{eq:bridge-error-appendix}
\end{equation}
\end{definition}

\subsection{Future Mean Under One Interaction}

\begin{lemma}[One-reset future mean]
\label{lem:one-reset-future-mean}
Under Assumption~\ref{ass:mode-wise-linear} and the one-reset setting, the target representation $k$ passive steps after
the interaction has mean
\begin{equation}
    \mu_k
    =
    A_0^{\tau+1+k}\psi_t
    +
    A_0^k\delta_\tau .
    \label{eq:one-reset-future-mean}
\end{equation}
\end{lemma}

\begin{proof}
Before the interaction, the target follows passive dynamics for $\tau$ steps:
\begin{equation}
    \psi_{t+\tau}
    =
    A_0^\tau\psi_t .
    \label{eq:pre-interaction-passive}
\end{equation}
At the interaction step, Definition~\ref{def:bridge-shift} gives
\begin{equation}
    \psi_{t+\tau+1}
    =
    A_0\psi_{t+\tau}
    +
    \delta_\tau .
    \label{eq:interaction-update-with-shift}
\end{equation}
After the interaction, the target follows passive dynamics for $k$ more steps:
\begin{equation}
    \psi_{t+\tau+1+k}
    =
    A_0^k
    \left(
        A_0\psi_{t+\tau}
        +
        \delta_\tau
    \right).
    \label{eq:post-interaction-passive}
\end{equation}
Substituting Eq.~\eqref{eq:pre-interaction-passive} into Eq.~\eqref{eq:post-interaction-passive}
yields
\begin{align}
    \psi_{t+\tau+1+k}
    &=
    A_0^k
    \left(
        A_0A_0^\tau\psi_t
        +
        \delta_\tau
    \right) \\
    &=
    A_0^{\tau+1+k}\psi_t
    +
    A_0^k\delta_\tau .
\end{align}
This proves Eq.~\eqref{eq:one-reset-future-mean}.
\end{proof}

Lemma~\ref{lem:one-reset-future-mean} shows the key structure: the interaction bridge shift is
not observed at the endpoint as $\delta_\tau$, but as the propagated quantity $A_0^k\delta_\tau$.

\subsection{Propagation of Bridge Error}

\begin{lemma}[Propagated bridge error]
\label{lem:propagated-bridge-error}
If the learned bridge shift has error $e=\widehat{\delta}_\tau-\delta_\tau$, then the induced endpoint
mean error after $k$ passive steps is
\begin{equation}
    \widehat{\mu}_k-\mu_k
    =
    A_0^k e .
    \label{eq:future-mean-error}
\end{equation}
Consequently,
\begin{equation}
    \|\widehat{\mu}_k-\mu_k\|
    \le
    \|A_0^k\|\|e\|.
    \label{eq:future-mean-error-bound}
\end{equation}
\end{lemma}

\begin{proof}
By Lemma~\ref{lem:one-reset-future-mean}, the true endpoint mean is
\begin{equation}
    \mu_k
    =
    A_0^{\tau+1+k}\psi_t
    +
    A_0^k\delta_\tau .
    \label{eq:true-endpoint-mean}
\end{equation}
Using the learned bridge shift $\widehat{\delta}_\tau$ gives
\begin{equation}
    \widehat{\mu}_k
    =
    A_0^{\tau+1+k}\psi_t
    +
    A_0^k\widehat{\delta}_\tau .
    \label{eq:learned-endpoint-mean}
\end{equation}
Subtracting Eq.~\eqref{eq:true-endpoint-mean} from Eq.~\eqref{eq:learned-endpoint-mean} gives
\begin{align}
    \widehat{\mu}_k-\mu_k
    &=
    A_0^k
    \left(
        \widehat{\delta}_\tau-\delta_\tau
    \right) \\
    &=
    A_0^k e .
\end{align}
Taking norms and using submultiplicativity gives Eq.~\eqref{eq:future-mean-error-bound}.
\end{proof}

\subsection{Energy Error Induced by Bridge Error}

We now translate the propagated endpoint mean error into an energy error.
For the appendix proof, we use the whitened local Gaussian energy
\begin{equation}
    E_k
    =
    -\frac{1}{2}
    \left\|
        \psi(y_f)-\mu_k
    \right\|^2
    +
    \mathrm{const}.
    \label{eq:gaussian-energy-appendix}
\end{equation}
This is the isotropic form of the local Gaussian score used in the main analysis.
A non-isotropic covariance only changes the norm and introduces constant conditioning factors, which
are suppressed in the main-paper proportionality statement.

\begin{assumption}[Normalized local endpoint region]
\label{ass:normalized-local-region}
The sampled endpoint lies in the local quadratic region of the correct mean:
\begin{equation}
    \left\|
        \psi(y_f)-\mu_k
    \right\|
    \le
    1 .
    \label{eq:normalized-local-region}
\end{equation}
\end{assumption}

This normalization is only for presentation.
If the local residual is bounded by a constant other than $1$, the first term in the final bound is
multiplied by that constant.

\begin{lemma}[Exact energy perturbation]
\label{lem:exact-energy-perturbation}
Suppose the propagated bridge-error relation in Lemma~\ref{lem:propagated-bridge-error} holds.
Then the energy perturbation caused by bridge error is
\begin{equation}
    \widehat{E}_k-E_k
    =
    \left(
        \psi(y_f)-\mu_k
    \right)^\top
    A_0^k e
    -
    \frac{1}{2}
    \left\|
        A_0^k e
    \right\|^2 .
    \label{eq:energy-perturbation-appendix}
\end{equation}
\end{lemma}

\begin{proof}
By Lemma~\ref{lem:propagated-bridge-error},
\begin{equation}
    \widehat{\mu}_k
    =
    \mu_k+A_0^k e .
    \label{eq:mean-with-propagated-error}
\end{equation}
Using Eq.~\eqref{eq:gaussian-energy-appendix}, the learned energy is
\begin{equation}
    \widehat{E}_k
    =
    -\frac{1}{2}
    \left\|
        \psi(y_f)-\widehat{\mu}_k
    \right\|^2
    +
    \mathrm{const}.
\end{equation}
Substituting Eq.~\eqref{eq:mean-with-propagated-error} gives
\begin{equation}
    \widehat{E}_k
    =
    -\frac{1}{2}
    \left\|
        \psi(y_f)-\mu_k-A_0^k e
    \right\|^2
    +
    \mathrm{const}.
\end{equation}
Therefore,
\begin{align}
    \widehat{E}_k-E_k
    &=
    -\frac{1}{2}
    \left[
        \left\|
            \psi(y_f)-\mu_k-A_0^k e
        \right\|^2
        -
        \left\|
            \psi(y_f)-\mu_k
        \right\|^2
    \right] \\
    &=
    \left(
        \psi(y_f)-\mu_k
    \right)^\top
    A_0^k e
    -
    \frac{1}{2}
    \left\|
        A_0^k e
    \right\|^2 .
\end{align}
This proves Eq.~\eqref{eq:energy-perturbation-appendix}.
\end{proof}

\begin{corollary}[Endpoint energy-error bound]
\label{cor:endpoint-energy-error-bound}
Suppose Lemma~\ref{lem:exact-energy-perturbation} holds and Assumption~\ref{ass:normalized-local-region} holds.
Then
\begin{equation}
    |\widehat{E}_k-E_k|
    \le
    \|A_0^k\|\|e\|
    +
    \frac{1}{2}
    \|A_0^k\|^2\|e\|^2 .
    \label{eq:energy-error-bound-appendix}
\end{equation}
\end{corollary}

\begin{proof}
Starting from Lemma~\ref{lem:exact-energy-perturbation}, apply the triangle inequality:
\begin{align}
    |\widehat{E}_k-E_k|
    &\le
    \left|
        \left(
            \psi(y_f)-\mu_k
        \right)^\top
        A_0^k e
    \right|
    +
    \frac{1}{2}
    \left\|
        A_0^k e
    \right\|^2 .
    \label{eq:triangle-bound}
\end{align}
By Cauchy--Schwarz,
\begin{equation}
    \left|
        \left(
            \psi(y_f)-\mu_k
        \right)^\top
        A_0^k e
    \right|
    \le
    \left\|
        \psi(y_f)-\mu_k
    \right\|
    \left\|
        A_0^k e
    \right\|.
\end{equation}
Using Assumption~\ref{ass:normalized-local-region},
\begin{equation}
    \left|
        \left(
            \psi(y_f)-\mu_k
        \right)^\top
        A_0^k e
    \right|
    \le
    \left\|
        A_0^k e
    \right\|.
\end{equation}
Finally, by submultiplicativity,
\begin{equation}
    \left\|
        A_0^k e
    \right\|
    \le
    \|A_0^k\|\|e\|.
\end{equation}
Substituting these bounds into Eq.~\eqref{eq:triangle-bound} gives
\begin{equation}
    |\widehat{E}_k-E_k|
    \le
    \|A_0^k\|\|e\|
    +
    \frac{1}{2}
    \|A_0^k\|^2\|e\|^2 .
\end{equation}
\end{proof}

Corollary~\ref{cor:endpoint-energy-error-bound} is the formal version of Proposition~\ref{prop:endpoint-energy-error}.
Removing constants associated with local normalization and whitening gives the main-paper scaling:
\[
    \sup |\widehat{E}_k-E_k|
    \propto
    \|A_0^k\|\|e\|
    +
    \frac{1}{2}
    \|A_0^k\|^2\|e\|^2 .
\]

\subsection{Interpretation}

The proof separates the endpoint error into two factors.
The first factor, $\|e\|$, is the local error in estimating the interaction bridge shift.
This error becomes large when few training samples directly cover the interaction boundary.
The second factor, $\|A_0^k\|$, is the passive post-interaction propagation factor.
Even if the bridge error is small, passive dynamics can rotate, attenuate, or amplify the error before
the endpoint is used in the contrastive objective.

This explains why endpoint-only CRL can learn fragile energies near contact.
If the positive future is sampled long after the interaction, CRL supervises the interaction action only
through the propagated term $A_0^k\delta_\tau$.
Interaction-weighted resampling targets the two terms in the bound: it increases bridge-centered
samples, reducing $\|e\|$, and it favors endpoints closer to the interaction, reducing the effective
post-interaction propagation length $k$.

\begin{figure}[h]
    \vspace{-8pt}   
    \centering
    \small
\includegraphics[width=1.0\linewidth]{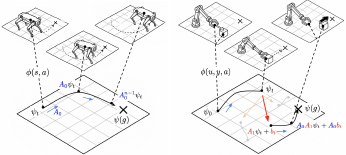}
    \caption{\textbf{Interaction-induced mode changes break the single-operator CRL geometry}. In smooth locomotion domains, future representations can often be organized by repeatedly applying a single temporal operator, producing a coherent trajectory toward the goal. In manipulation, interactions such as contact or object transfer switch the active dynamics and introduce residual displacements that depend on the actuated factor and action. Hence, the future distribution becomes piecewise and branching, making standard CRL smoothing over critical interaction transitions.}
    \label{fig:theory_illustration}
    \vspace{-.5cm}
\end{figure}

%% file: appendix/related_work.tex
\section{Extended Related Work} \label{app: rl}
This work both formalizes the challenges when applying contrastive reinforcement learning to manipulation and proposes a sampling-based strategy for leveraging interactions to mitigate those challenges. This places it in the intersection of work investigating CRL and robot learning for manipulation using interactions. 

\textbf{Contrastive Reinforcement Learning}: CRL is the application of contrastive learning~\citep{oord2018representation} to goal-conditioned reinforcement learning (GCRL), both as a classification~\citep{eysenbach2021clearning} and regression objective ~\citep{eysenbach2022contrastive}.  Recent work has expanded this towards metric learning~\citep{zheng2025multistep,myers2026offline}, planning as time series modeling~\citep{eysenbach2024inference}, language alignment~\citep{myers2023goal,myers2026temporal} and combinatorial reasoning and search~\citep{wang2026temporal,ziarko2026contrastive}. \methodname{} bears the closest similarity to~\citet{ziarko2026contrastive}, which also employs a resampling procedure, but \methodshort{} leverages interactions to target the resampling towards difficult-to-model dynamics, while~\citet{ziarko2026contrastive} utilizes different contexts for in-trajectory sampling. We also build on the successes of CRL when applied in simulated manipulation~\citep{wang20261000}, especially related to emergent exploration~\citep{liu2025single,bastankhah2026demystifying}, leveraging these properties which allow for limited manipulation performance in a sample-efficient manner relative to prior GCRL work~\citep{puterman1990markov,kaelbling1993learning,liu2022goal,andrychowicz2017hindsight,bai2021addressing}. Besides contrastive learning as applied to GCRL, \methodshort{} draws on ideas from general contrastive learning, especially related to leveraging privileged information such as class labels (wherein our case we leverage interactions and in-trajectory sampling)~\citep{feng2022adaptive,denize2023similarity,khosla2020supervised,hoffmann2022ranking}, and debiasing strategies which carefully select negative examples, which we adapt for interaction-based sampling~\citep{chuang2020debiased,huynh2022boosting,dwibedi2021little}. Our observations about the relationship between non-smooth dynamics and smooth representations learned by contrastive methods are in part inspired by~\citet{betser2026infonce}, which observes that the distribution of representations tends towards a thin-shell Gaussian. Finally, recent work has shown close connections between contrastive RL methods and the general class of methods used in unsupervised Reinforcement Learning~\citep{agarwal2025unified}, including learning to represent all rewards through successor representations~\citep{sikchi2026rlzero,zheng2026can,touati2021learning,agarwal2024proto,zheng2024contrastive}, mutual information skill learning~\citep{modirshanechi2026unifying,levy2025unsupervised}, representation learning~\citep{ferns2011bisimulation,zhang2020learning,rudolph2024learning} and world models~\citep{farebrother2026compositional}. However, the application of factorization and especially entity interactions to the methods based on contrastive learning, which we use in \methodshort{}, remains limited. 

\textbf{Manipulation Learning with Interactions}: Robot manipulation encompasses a wide range of tasks, from grasping and moving objects to striking, inserting or throwing~\citep{han2023survey}. Interactions~\citep{chuck2024automated} have been incorporated into learning for manipulation in a variety of ways, including controllability~\citep{seitzer2021causal}, data augmentation~\citep{pitis2020counterfactual,pitis2022mocoda}, hierarchical reinforcement learning~\citep {chuck2020hypothesis,chuck2023granger,wang2024skild, feng2026learning}, skill learning~\citep{hu2024disentangled,rodriguez2025pixels,hosseini2026susd}, causal modeling~\citep{lee2023scale,biswas2024gaze, feng2023learning} and exploration~\citep{wang2023elden}. This work introduces a novel formalization for the challenges of robot learning for manipulation as mode-changing dynamics, and \methodshort{} is a novel algorithm for incorporating interaction information into the training of contrastive RL. \methodshort{} shows significantly improved sample efficiency, and the first application to robot air hockey~\citep{chuck2024air}. Other work in manipulation that is more loosely leveraging interactions includes work centered around modeling dynamics, such as through graph networks~\citep{sieb2020graph,huang2023planning,lin2022efficient,zhang2025particle,huangrobocraft,kedia2024interact}, (which incorporate interactions through the message passing framework), human conditional intent~\citep{zhang2024adaptigraph} or human video~\citep{singh2025hand,zhu2026vision}, and affordances~\citep{khazatsky2021can,girgin2024multiobject}. Most closely related is the work investigating dynamics modeling~\citep{haramati2026hierarchical,wang2026factored} which learn sturctures that incorporate factors for policy learning, often also identifying the factors themselves from observations~\citep{daniel2026latent, qi2025ec}. \methodshort{} is applied to CRL, which learns structure for GCRL policies, but these model-learning algorithms are likely to be effective in concert with CRL methods, where the dynamics model and policies can work together.

%% file: appendix/iwr_details.tex
\section{\methodname{} Details}
\label{app:iwr_details}
In this section, we provide some description to give a better understanding of how IWR addresses the challenge of mode switching.

IWR does not change the CRL loss or its energy landscape directly. Instead, it changes the effective sampling distribution of contrastive updates. This provides an inductive bias toward mode-switching regions, giving more critic updates to interaction-relevant transitions rather than to easy passive motion or object-reaching behavior. In terms of our analysis (Section~\ref{sec: analysis}), IWR increases the probability that sampled positive futures cross the interaction bridge near the latent interaction time \(\tau_t\). Such samples expose the interaction shift directly, instead of observing it only after long passive propagation through powers of \(A_0\). Intuitively, based on the analysis of intrinsic reward \cite{bastankhah2026demystifying}, our paradigm shifts intrinsic reward-driven exploration from ``moving the object to the future'' to ``learning how to control the object to move to the future.'' This also connects to Proposition~\ref{prop:endpoint-energy-error} as increasing the coverage of interaction-centered samples reduces the error contribution from the interaction residual. Exact bridge samples provide direct supervision for the piecewise jump term, while the soft weighting in Eq.~\ref{eq:iwr} also covers nearby positives and reduces error accumulation from repeatedly propagating the interaction signal through \(A_0^k\).
Consequently, IWR reduces the passive bias of ordinary future sampling and improves the conditioning of interaction-relevant representation learning.

% We preserve the CRL loss landscape while adding an inductive bias that yields more updates on the interaction part, rather than on easy-to-learn object-reaching behavior or uncontrollable passive target-object motion.
% Intuitively, based on the analysis of intrinsic reward \cite{bastankhah2026demystifying}, our paradigm shifts intrinsic reward-driven exploration from ``moving the object to the future'' to ``learning how to control the object to move to the future.''
% Based on Proposition~\ref{prop:endpoint-energy-error}, a higher coverage of the interaction term would control the error bound $|\widehat{E}_k - E_k|$
% . On the one hand, exact interaction positives would help estimate the piecewise jump transfer by controlling the error bound of $\|e\|$.
% On the other hand, soft weighting label would sample positives close to the interaction point. Then we could reduce the accumulation of errors from the exponential iteration of $\|A_0^k\|$.
% As a result, interaction control would be more consistent and robust.

%% file: appendix/training_setup.tex
\section{Training Setup}
\label{app:training_setup}

% Appendix snippet. Expected packages in the main paper:
% \usepackage{booktabs}
% \usepackage{array}
% \usepackage{amsmath,amssymb}

\subsection{Tasks and Algorithm Grid}

Paper experiments use five seeds $\{7,8,9,10,11\}$.
Experiment grid contains three simulation backends and seven algorithms.

\begin{table}[h]
\centering
\small
\setlength{\tabcolsep}{5pt}
\caption{Paper-core benchmark grid.}
\label{tab:appendix-paper-core-grid}
\begin{tabular}{p{0.18\linewidth}p{0.64\linewidth}c}
\toprule
Backend & Tasks & Count \\
\midrule
Air Hockey & puck-goal position, $r=0.06$, observation mode=velocity / real-transfer-history & 2 \\
Box2D & center, goal, hard, hard-velocity, maze & 5 \\
MetaWorld & push, pick-place, peg-insert-side, sweep-into & 4 \\
\midrule
Algorithms & PPO, SAC, SAC+HER, SAC+HINT, CRL, CRTR, IWR & 7 \\
\bottomrule
\end{tabular}
\end{table}

% \begin{table}[h]
% \centering
% \small
% \setlength{\tabcolsep}{5pt}
% \caption{Algorithm identities used in the paper-core launcher.}
% \label{tab:appendix-algorithm-identities}
% \begin{tabular}{p{0.16\linewidth}p{0.22\linewidth}p{0.50\linewidth}}
% \toprule
% Paper name & Code id & Main training difference \\
% \midrule
% PPO & \texttt{ppo} & Sparse-success PPO baseline. \\
% SAC & \texttt{sac} & Sparse-success SAC baseline. \\
% SAC+HER & \texttt{sac\_her} & SAC with uniform future-goal HER. \\
% SAC+HINT & \texttt{sac\_hint} & SAC+HER with interaction-weighted HER anchors. \\
% CRL & \texttt{sgcrl} & Original SGCRL objective and replay sampling. \\
% CRTR & \texttt{crtr} & CRTR-style repeated episode contexts, $R=8$. \\
% IWR (ours) & \texttt{iwr} & Interaction-weighted replay, $R=8$. \\
% \bottomrule
% \end{tabular}
% \end{table}

\subsection{Shared Training Defaults}

% All paper-core runs disable inline evaluation and rollout videos during
% training.  Evaluation curves and post-hoc evaluations are computed from saved
% checkpoints.  The configured eval environment counts remain in the launch
% environment for compatibility, but \texttt{--disable-eval} prevents evaluator
% construction.

Shared training details can be found in Table~\ref{tab:appendix-shared-training}

\begin{table}[h]
\centering
\small
\setlength{\tabcolsep}{5pt}
\caption{Shared Training Defaults}
\label{tab:appendix-shared-training}
\begin{tabular}{ll}
\toprule
Setting & Value \\
\midrule
Total environment steps & $10{,}000{,}000$ \\
Train environments & $256$ \\
Batch size & $64$ \\
Replay capacity & $200{,}000$ \\
Minimum replay size & $1{,}024$ \\
Random exploration warmup & $200{,}000$ \\
Updates per environment iteration & $16$ \\
% Checkpoint interval & $200{,}000$ environment steps \\
Network architecture & MLP with hidden sizes $(512,512)$ \\
Representation dimension & $256$ for CRL-family critics \\
Future discount for future sampling & $\gamma_f=0.99$ \\
Random actor-goal fraction & $0.5$ \\
Critic type & CPC \\
Log-sum-exp regularizer coefficient & $0.01$ \\
Actor learning rate & $3\times 10^{-4}$ \\
Critic learning rate & $3\times 10^{-4}$ \\
Optimizer & Adam \\
% Inline evaluation & Disabled \\
% Training videos & Disabled \\
\bottomrule
\end{tabular}
\end{table}

\begin{table}[h]
\centering
\small
\setlength{\tabcolsep}{5pt}
\caption{Backend-specific training settings.}
\label{tab:appendix-backend-training}
\begin{tabular}{p{0.18\linewidth}p{0.18\linewidth}p{0.15\linewidth}p{0.34\linewidth}}
\toprule
Backend & Eval Envs & Termination \\
\midrule
Air Hockey & $128$ & First Success or 200 ticks \\
Box2D & $128$ & 200 ticks \\
MetaWorld & $32$ & First Success or 200 ticks \\
\bottomrule
\end{tabular}
\end{table}

\subsection{CRL-Family Objective}

For CRL, CRTR, and IWR, the actor and critic are identical except for replay
sampling.  Given anchor state-action pair $(s_i,a_i)$ and sampled future goal
$g_i^+$, the critic score is
\[
    E_\theta(s_i,a_i,g_j)
    =
    \phi_\theta(s_i,a_i)^\top \psi_\theta(g_j),
\]
with cosine-normalized score on Air Hockey and dot-product score on Box2D and
MetaWorld.  The in-batch CPC objective is
\[
    \begin{aligned}
    \mathcal L_{\mathrm{CRL}}
    &=
    -\frac{1}{B}
    \sum_{i=1}^{B}
    \log
    \frac{\exp(E_\theta(s_i,a_i,g_i^+))}
    {\sum_{j=1}^{B}\exp(E_\theta(s_i,a_i,g_j^+))}
    \\
    &\quad+
    \lambda_{\mathrm{lse}}
    \frac{1}{B}
    \sum_i
    \left(
        \log\sum_j \exp(E_\theta(s_i,a_i,g_j^+))
    \right)^2 ,
    \end{aligned}
\]
where $B=64$ and $\lambda_{\mathrm{lse}}=0.01$.

Future goals are sampled from the same episode.  For anchor time $t$, offsets
$\Delta\geq 1$ are sampled with weights proportional to
\[
    \gamma_f^{\Delta-1},
    \qquad
    \gamma_f = 0.99,
\]
subject to the remaining episode length.

% \subsection{CRTR Replay Sampling}

% CRTR keeps the same CRL loss and critic.  The only change is the batch
% construction.  Episodes are sampled with probability proportional to stored
% episode length.  A smaller set of base episodes is selected, and each base
% episode is repeated $R$ times inside the same batch:
% \[
%     R = 8.
% \]
% With batch size $64$, this gives $8$ base episodes per update and $8$ rows per
% episode context.  Anchor steps and future goals are sampled independently
% within each repeated episode context using the same discounted future sampler
% as CRL.  Unless explicitly marked as an ablation, CRTR in paper-core means
% \texttt{crtr-context-mode=episode} and
% \texttt{crtr-context-repetition-factor=8}.

\subsection{IWR Replay Sampling}

We calculate the sample weight with $w = \epsilon + \exp(-\frac{1}{2\sigma^2}\|d-c\|_2)$ with parameters in Table ~\ref{tab:appendix-iwr-params}. Note that the threshold $c$ we select is at least 1.5 times larger than the true distance of factor and target object. We only need a loose heuristic about where the exact interaction will happen.

\begin{table}[h]
\label{tab:iwr}
\centering
\small
\setlength{\tabcolsep}{5pt}
\caption{IWR geometry parameters.}
\label{tab:appendix-iwr-params}
\begin{tabular}{llll}
\toprule
Task group & $c$ & $2\sigma^2$ \\
\midrule
Air Hockey & $0.12$ & $30$ \\
MetaWorld & $0.09$ & $20$ \\
Box2D center / goal & $3.3$ & $100$ \\
Box2D hard / hard-velocity / maze & $2.0$ & $80$ \\
\bottomrule
\end{tabular}
\end{table}

\subsection{Sparse-Reward PPO and SAC Baselines}

PPO, SAC, SAC+HER, and SAC+HINT use the same goal-conditioned state and action
interfaces as the CRL-family methods.  They do not use dense simulator rewards.
The reward is the sparse success indicator $r_t=\mathbf{1}\{\text{success at }t\}$

\begin{table}[h]
\centering
\small
\setlength{\tabcolsep}{5pt}
\caption{PPO and SAC-family baseline settings.}
\label{tab:appendix-baseline-settings}
\begin{tabular}{p{0.16\linewidth}p{0.34\linewidth}p{0.38\linewidth}}
\toprule
Algorithm & Setting & Value \\
\midrule
PPO & rollout steps & $16$ \\
PPO & epochs per rollout & $4$ \\
PPO & clipping parameter & $0.2$ \\
PPO & GAE $\lambda$ & $0.95$ \\
PPO & value loss coefficient & $0.5$ \\
PPO & max gradient norm & $0.5$ \\
\midrule
SAC & target update rate $\tau$ & $0.005$ \\
SAC & entropy coefficient $\alpha$ & $0.2$ \\
SAC & discount & $0.99$ \\
SAC & critics & two scalar Q networks, MLP $(512,512)$ \\
SAC+HER & HER ratio & $0.8$ \\
SAC+HER & HER goal source & discounted future achieved goals \\
SAC+HINT & HER mode & interaction-weighted anchor sampling \\
\bottomrule
\end{tabular}
\end{table}

For SAC+HER, each sampled transition uses the original desired goal with
probability $0.2$ and a future achieved goal with probability $0.8$.  The
reward is recomputed against the relabeled goal.  SAC+HINT keeps the same HER
ratio but samples anchors using interaction weights.  For HINT, the anchor
weight is
\[
    w_t = \exp(-d_t/\sigma) + 10^{-3},
\]
where $d_t$ is the source-target distance.

\begin{table}[h]
\centering
\small
\setlength{\tabcolsep}{5pt}
\caption{SAC+HINT interaction parameters.}
\label{tab:appendix-sac-hint-params}
\begin{tabular}{llll}
\toprule
Backend & Contact distance & $\sigma$ \\
\midrule
Air Hockey & $0.12$ & $30$ \\
Box2D & $2.0$ & $80$ \\
MetaWorld & $0.09$ & $20$ \\
\bottomrule
\end{tabular}
\end{table}

\subsection{Evaluation Metrics}

Let one logging batch contain $N=256$ completed episodes,
matching the number of train environments.  For episode $e$, define
\[
    \begin{aligned}
    o_e &= \mathbf{1}\{\text{episode }e\text{ reached success at least once}\},\\
    z_e &= \sum_t \mathbf{1}\{\text{success at step }t\}.
    \end{aligned}
\]

For Box2d we report on $z_e$ and for Metaworld we report on $o_e$. On Table ~\ref{tab:main1}, Box2d result is normalized by 100 ticks. Since we repeat the experiment for 5 times, we first record the mean result every 200000 steps as checkpoints and report the best score for comparison.

\section{Simulation Environment Setup}
\label{app:simulation_setup}

% Appendix snippet. Expected packages in the main paper:
% \usepackage{booktabs}
% \usepackage{array}
% \usepackage{amsmath,amssymb}

\subsection{Simulation Setting}

\begin{figure}[t]
    \centering
    \begin{minipage}[t]{0.75\textwidth}
        \centering
        \vspace{0pt}
        \includegraphics[width=\linewidth]{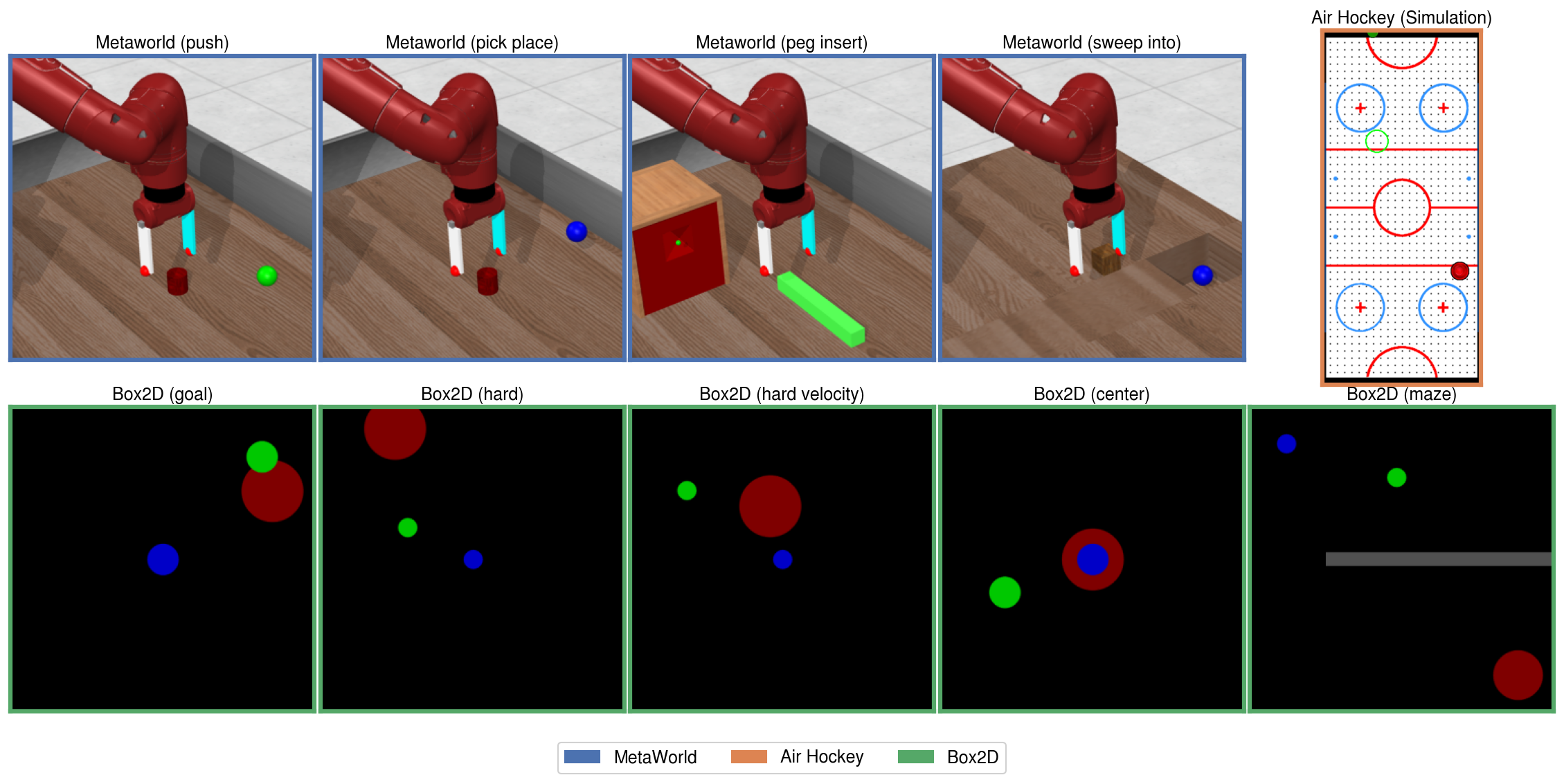}
        \caption{Task sets used in the simulation experiments.}
        \label{fig:task_set}
    \end{minipage}%
    \begin{minipage}[t]{0.25\textwidth}
        \centering
        \vspace{0pt}
        \includegraphics[width=\linewidth]{figure/air_hockey_setup.png}
        \caption{Real-world robotic air hockey testbed.}
        \label{fig:real_world_air_hockey_setup}
    \end{minipage}
    \vspace{-1em}
\end{figure}
All environments are wrapped as goal-conditioned vector environments.  Each
wrapper exposes a state vector $s_t$, achieved goal $g_t^{\mathrm{ach}}$,
desired goal $g_t^{\mathrm{des}}$, action $a_t$, and boolean success flag.  CRL
methods use future achieved goals for contrastive relabeling; sparse-reward
baselines use the same success flag as reward.  Episodes are automatically
reset in the vectorized training environment.

% \begin{table}[h]
% \centering
% \small
% \setlength{\tabcolsep}{5pt}
% \caption{Backend-level simulation settings.}
% \label{tab:appendix-backend-simulation}
% \begin{tabular}{p{0.24\linewidth}p{0.18\linewidth}p{0.22\linewidth}p{0.24\linewidth}}
% \toprule
% Backend & Train horizon & Termination on success & Vectorization \\
% \midrule
% Air Hockey & $300$ steps & yes & $256$ train envs \\
% Box2D & $400$ steps in paper launch & no & $256$ train envs \\
% MetaWorld & $150$ steps & yes & async, $256$ train envs \\
% \bottomrule
% \end{tabular}
% \end{table}

\subsection{Air Hockey}

The success radius from the goal state is $r_{\mathrm{goal}} = 0.06$. 
The paddle, puck, and goal are randomized at reset.  The wrapper terminates an
episode on success during training.  For the default physics tasks, episodes
also terminate if the puck hits the bottom or passes the paddle; out-of-bounds
termination is disabled.  For real-transfer, out-of-bounds termination is
enabled.

\begin{table}[h]
\centering
\small
\setlength{\tabcolsep}{5pt}
\caption{Default Air Hockey simulator parameters.}
\label{tab:appendix-air-hockey-default-physics}
\begin{tabular}{ll}
\toprule
Parameter & Value \\
\midrule
Simulator & Box2D Air Hockey simulator \\
Table length / width & $1.9304$ / $0.8636$ \\
Render size & $360$ \\
Paddle radius & $0.0508$ \\
Puck radius & $0.03175$ \\
Paddle density / damping & $2500$ / $3$ \\
Puck density / damping & $250$ / $0.5$ \\
Gravity & $-0.5$ \\
Force scaling & $1000$ \\
Max force timestep & $100$ \\
Wall bounce scale & $0.02$ \\
Blocks / obstacles / targets & $0/0/0$ \\
Paddles / pucks & $1/1$ \\
Randomized paddle spawn & true \\
Goal radius & $0.06$ \\
\bottomrule
\end{tabular}
\end{table}

\begin{table}[h]
\centering
\small
\setlength{\tabcolsep}{5pt}
\caption{Air Hockey real-transfer physics and corruption settings.}
\label{tab:appendix-air-hockey-real-transfer}
\begin{tabular}{ll}
\toprule
Parameter & Value \\
\midrule
Observation type & history \\
Domain randomization & enabled \\
Randomized variables & paddle density, puck damping, gravity \\
Paddle density range & $[2250,3750]$ \\
Puck damping range & $[0.1335,0.2225]$ \\
Gravity range & $[-0.826,-0.496]$ \\
SysID gravity & $-0.661$ \\
SysID puck damping & $0.178$ \\
Paddle damping & $17$ \\
Force scaling & $0.99$ \\
Max force timestep & $10000$ \\
Puck density & $3000$ \\
Paddle / puck radius & $0.0508$ / $0.03175$ \\
Paddle restitution & $1.0$ \\
Side / end wall restitution & $0.99$ / $0.7$ \\
PID control & enabled \\
PID gains & $k_p=9000$, $k_d=50$, $k_i=0$ \\
History length & $2$ \\
Puck observation noise & enabled, std. $0.01$ \\
Random occlusion & enabled, start rate $0.05$ per step \\
Occlusion length weights & $[75,39,18,9,4,2,1]$ \\
Observation delay & enabled, $0.025$s \\
Action delay & disabled \\
Near-paddle puck spawn probability & $0.15$ \\
Near-paddle spawn offset & $[0.025,0.05]$ m \\
Near-paddle spawn speed & $[0.0,0.2]$ m/s \\
Linear top spawn speed & $[0.0,0.5]$ m/s \\
\bottomrule
\end{tabular}
\end{table}

\subsection{Box2D}

% TODO

% The hard-velocity task is a local compatibility variant implemented in the
% wrapper by taking \texttt{goal-ctrl-target-hard} and setting the target initial
% velocity to $5.0$.
Box2D tasks are continuous-control 2D tasks with circular bodies.  It has 400 steps per episode and success does not terminate the episode in the Box2D launch. Table~\ref{app:box2d_setting} shows shared parameters. As visualized in Figure~\ref{fig:task_set}, Box2d aims to evaluate how well the agent could control the blue ball to punch the green ball into the goal position, and control it inside the goal position as long as possible. Initialization details could be seen in Table~\ref{app:box2d_tasks}.

\begin{table}[h]
\centering
\small
\setlength{\tabcolsep}{5pt}
\caption{Default Box2D simulator parameters.}
\label{tab:appendix-box2d-default}
\begin{tabular}{ll}
\toprule
Parameter & Value \\
\midrule
% Simulator & \texttt{ac\_infer} Box2D simulator \\
World length / width & $10.0$ / $10.0$ \\
Action space & continuous \\
Force scaling & $200$ \\
Wall bounce factor & $0.5$ \\
Control damping / density & $2.0$ / $0.75$ \\
Target damping / density & $1.0$ / $0.1$ \\
Ordinary control / target radius & $0.5$ / $0.5$ \\
Hard control / target radius & $0.3$ / $0.3$ \\
Goal success radius & $1.0$; maze uses $0.8$ \\
Render size & $80$; maze uses $128$ \\
% Target spawn on goal & rejected \\
% Default target initial velocity & $0.0$ \\
\bottomrule
\end{tabular}
\label{app:box2d_setting}
\end{table}

\begin{table}[h]
\centering
\small
\setlength{\tabcolsep}{5pt}
\caption{Box2D tasks initialization description}
\label{tab:appendix-box2d-tasks}
\begin{tabular}{p{0.20\linewidth}p{0.28\linewidth}p{0.14\linewidth}p{0.28\linewidth}}
\toprule
Paper task name & Reset / goal initialization \\
\midrule
Box2D center & random bodies, goal fixed in the center \\
Box2D goal & random control, target, and goal \\
Box2D hard & random control, target, and goal \\
Box2D hard velocity & hard task with target initial velocity $5.0$ \\
Box2D maze & control fixed at $(-3.8,-3.8)$, goal fixed at $(3.8,3.8)$, wall fixed at $(x,y,w,h)=(0,2.5,0.22,5.0)$ \\
\bottomrule
\end{tabular}
\label{app:box2d_tasks}
\end{table}

\subsection{MetaWorld}

MetaWorld tasks use compact point observations.  The state dimension is $7$:
\[
    s =
    [p_{\mathrm{hand}}\in\mathbb{R}^3,\;
      q_{\mathrm{gripper}}\in\mathbb{R},\;
      p_{\mathrm{object}}\in\mathbb{R}^3].
\]
Default goal gripper opening is $0.4$ and default hand offset is
$(0,0,0.03)$.  The default minimum distance
between the initial object and sampled goal is $0.15$.  Episodes terminate on
success during training.

\begin{table}[h]
\centering
\small
\setlength{\tabcolsep}{5pt}
\caption{MetaWorld tasks and success definitions.}
\label{tab:appendix-metaworld-tasks}
\begin{tabular}{p{0.23\linewidth}p{0.23\linewidth}p{0.22\linewidth}p{0.26\linewidth}}
\toprule
Paper task name & Env id & Achieved object & Success condition \\
\midrule
MetaWorld push & \texttt{push-v2} & object position & distance $<0.05$ \\
MetaWorld pick place & \texttt{pick-place-v2} & object position & distance $<0.07$ \\
MetaWorld peg insert & \texttt{peg-insert-side-v2} & peg head position & scaled distance $<0.07$ \\
MetaWorld sweep into & \texttt{sweep-into-v2} & object position & distance $<0.05$ and $z$ error $<0.02$ \\
\bottomrule
\end{tabular}
\end{table}

The \texttt{peg-insert-side-v2} wrapper maps to the canonical
\texttt{peg-insert-side-v3} task, uses success scale $(1,2,2)$, and uses a
goal hand offset $(0.13,0,0.03)$.  The \texttt{sweep-into-v2} success
calculation ignores the target $z$ axis for the main distance while enforcing a
separate $z$-axis threshold of $0.02$.

\subsection{Post-Hoc Air Hockey Consecutive Evaluation}

For the one-minute Air Hockey challenge, evaluation is separate from training.
The evaluator disables terminate-on-success so the puck can be re-armed to a
new goal after each success.  The standard setting is
\[
    \text{horizon}=1200\text{ steps},\qquad
    \text{seeds}=1,\ldots,16.
\]
The recorded metrics are complete rate, mean number of consecutive goals, and
maximum number of consecutive goals.  Rendering is disabled by default for
batch evaluation, but the same evaluator can render rollout videos for selected
checkpoints.

%% file: appendix/real_setup.tex
\section{Real Setup}
\label{app:real_setup}
The physical setup utilizes a UR5e Universal Robotics robotic arm with a Robotiq two-fingered gripper holding an air hockey paddle. The robot is placed at one end of an inclined air hockey table, such that the puck will always fall back to the robot. By utilizing two 3D-printed curved side pieces, the robot can self-reset by pushing the puck up the side, then striking it. We use two human-demonstrated open-loop policies to perform the puck reset. The puck positions are detected using a 120Hz overhead camera with blob detection and converted with a homography into x,y coordinates. The full control loop runs at 20Hz, using the last 5 puck positions, the paddle's position and velocity, and the goal position, which is the same as in the simulation.

% \begin{wrapfigure}{r}{0.40\textwidth}
%     \centering
%     \vspace{-1em}
%     \includegraphics[width=0.38\textwidth]{figure/air_hockey_setup.png}
%     \caption{
%     Real-world robotic air hockey testbed.
%     }
%     \label{fig:real_world_air_hockey_setup}
%     \vspace{-1em}
% \end{wrapfigure}

Trajectories are as follows: after the puck successfully crosses the midline of the table, the learned policy begins until the puck falls to a position lower (closer to the base) than the robot. This pattern is akin to juggling, except that the task requires putting the radius of the puck inside the radius of the goal position (success), before the puck falls past the paddle (failure). Evaluation utilizes a fixed set of 20 goals in a $4\times5$ pattern, with 5 goals per row and 4 rows getting progressively lower. Additional details, including the dimensions of the table, paddle, puck, etc., can be found in \citet{chuck2024robot}.

To transfer trained policies from the simulator to the real robot, we first identified system parameters for the paddle using a small number (2 min) of human-generated paddle data. Puck parameters are then estimated based on human juggling data and the physics of the puck, as this was superior to black box optimization for estimating these parameters. Consequently, a meaningful sim-to-real gap exists in the puck movement between the real and simulated environments. To train policies robust to this gap, we leverage domain randomization in the style of \citet{tobin2017domain,peng2018sim,kumar2021rma}. For system parameters, we randomize the paddle density (which affects collisions), puck damping, and gravity as follows: 
\begin{center}
\small
\begin{tabular}{lccc}
\hline
\textbf{Parameter} & \textbf{Range} & \textbf{SysID} & \textbf{Variation} \\
\hline
Paddle density & $[2250.0,\ 3750.0]$ & $3000$ & $\pm 25\%$ \\
Puck damping & $[0.1335,\ 0.2225]$ & $0.178$ & $\pm 25\%$ \\
Gravity & $[-0.826,\ -0.496]$ & $-0.661$ & $\pm 25\%$ \\
\hline
\end{tabular}
\end{center}

We also perform substantial random resets for the puck reset, resetting the puck at positions across the top and middle with randomly selected velocities, as well as random positions near the paddle, to simulate the variance and challenging positions from where a trajectory might start. The puck is also spawned in random positions across its workspace. Random Noise and occlusions (missing values) are added to the observation of the puck position to force the policy to be robust to perceptual noise, as well as adding perceptual delay, where the puck position is lagged when given to the agent. Finally, 4-step smoothing is applied to the real robot, and the movement of the paddle in simulation is tuned with a PID controller to match this kind of smooth motion. 

Trajectory visualization could be found in: IWR: Figure~\ref{fig:iwr-render-paired}, PPO: Figure~\ref{fig:ppo-paired}, SAC: Figure~\ref{fig:sac-gcrl-paired}, SAC+HER: Figure~\ref{fig:sac-her-paired}, SAC+HINT: Figure~\ref{fig:sac-hint-paired}, SGCRL: Figure~\ref{fig:sgcrl-render-paired}, CRTR: Figure~\ref{fig:crtr-paired}.

% -----------------------------
% IWR Render
% -----------------------------
\begin{figure}[t]
    \centering
    \begin{subfigure}[t]{0.42\textwidth}
        \centering
        \includegraphics[width=\linewidth]{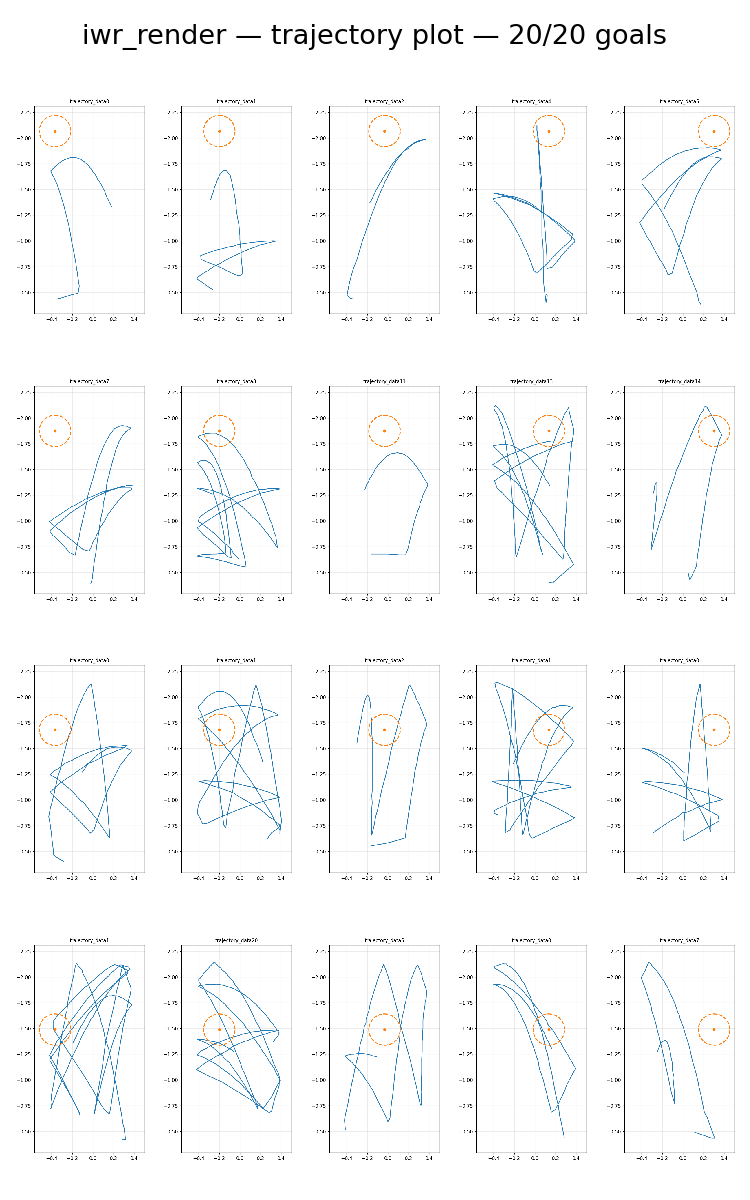}
        \caption{Trajectory-space visualization.}
        \label{fig:iwr-render-trajectory}
    \end{subfigure}
    \hfill
    \begin{subfigure}[t]{0.42\textwidth}
        \centering
        \includegraphics[width=\linewidth]{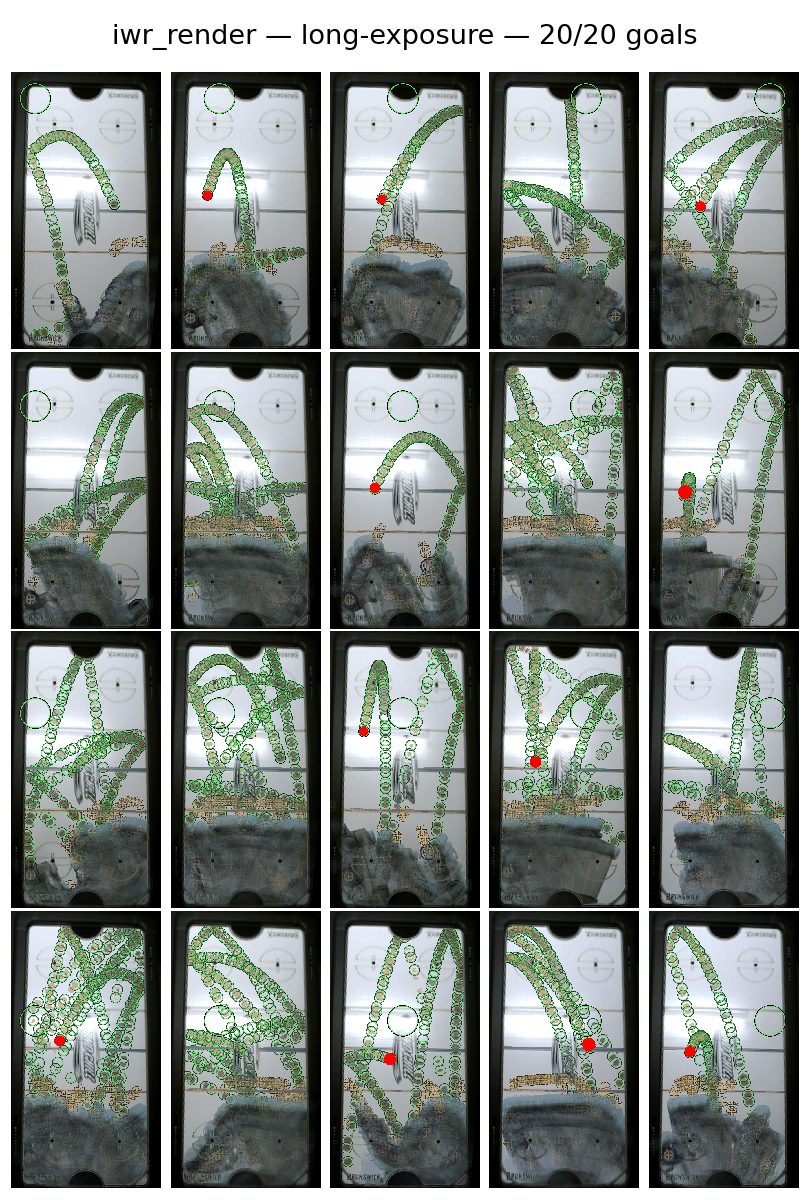}
        \caption{Long-exposure real-world rollout.}
        \label{fig:iwr-render-long-exposure}
    \end{subfigure}
    \caption{IWR evaluation rollouts across 20/20 achieved goals.}
    \label{fig:iwr-render-paired}
\end{figure}

% -----------------------------
% PPO
% -----------------------------
\begin{figure}[t]
    \centering
    \begin{subfigure}[t]{0.42\textwidth}
        \centering
        \includegraphics[width=\linewidth]{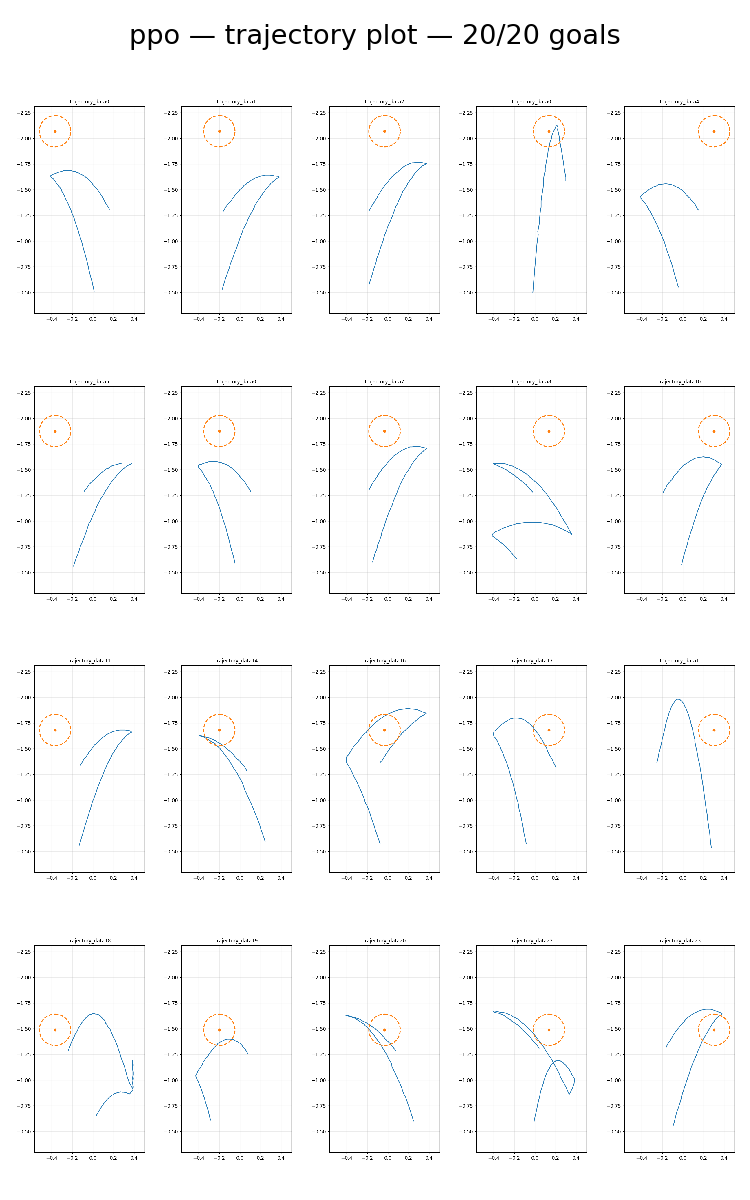}
        \caption{Trajectory-space visualization.}
        \label{fig:ppo-trajectory}
    \end{subfigure}
    \hfill
    \begin{subfigure}[t]{0.42\textwidth}
        \centering
        \includegraphics[width=\linewidth]{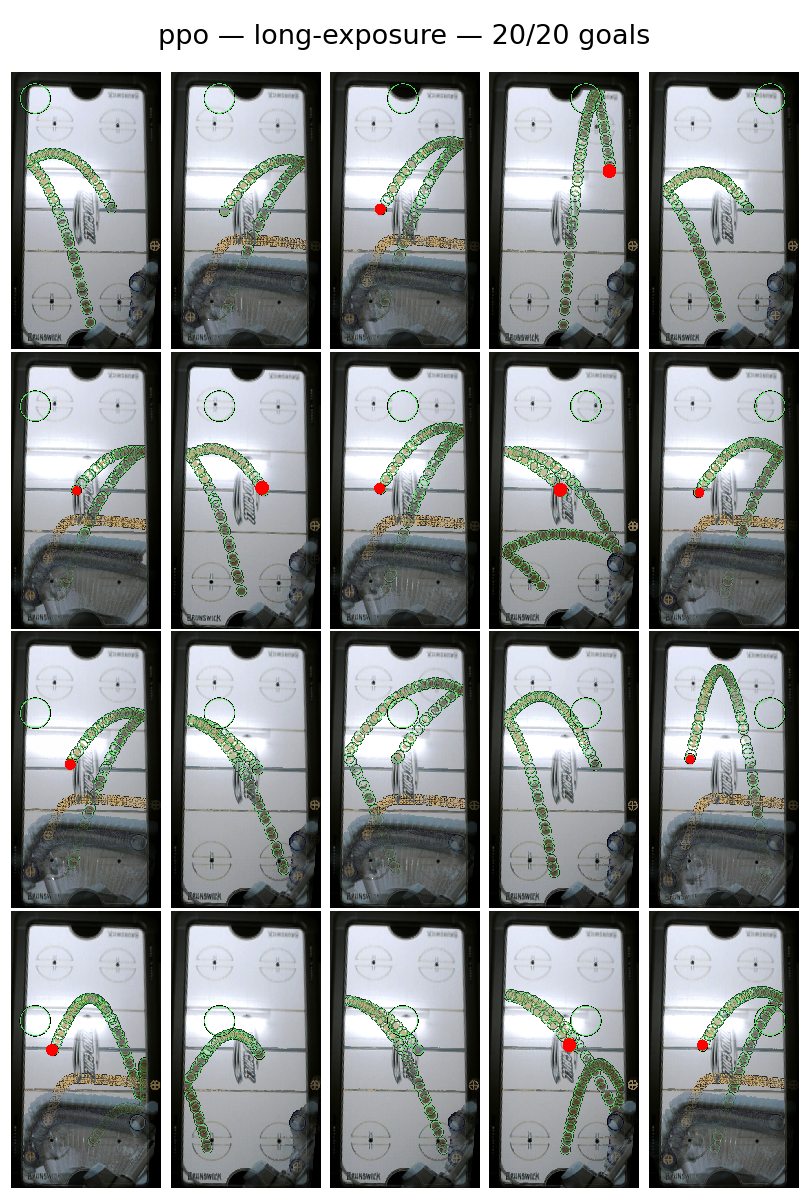}
        \caption{Long-exposure real-world rollout.}
        \label{fig:ppo-long-exposure}
    \end{subfigure}
    \caption{PPO evaluation rollouts across 20/20 achieved goals.}
    \label{fig:ppo-paired}
\end{figure}

% -----------------------------
% SAC-GCRL
% -----------------------------
\begin{figure}[t]
    \centering
    \begin{subfigure}[t]{0.42\textwidth}
        \centering
        \includegraphics[width=\linewidth]{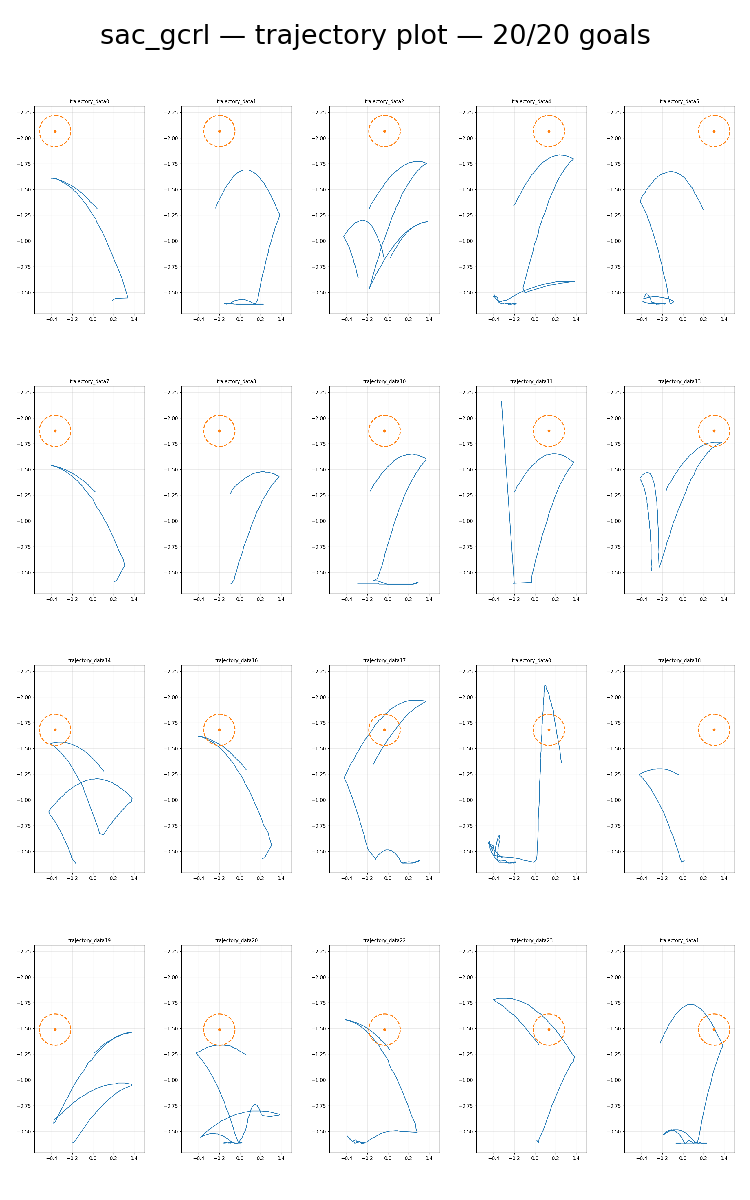}
        \caption{Trajectory-space visualization.}
        \label{fig:sac-gcrl-trajectory}
    \end{subfigure}
    \hfill
    \begin{subfigure}[t]{0.42\textwidth}
        \centering
        \includegraphics[width=\linewidth]{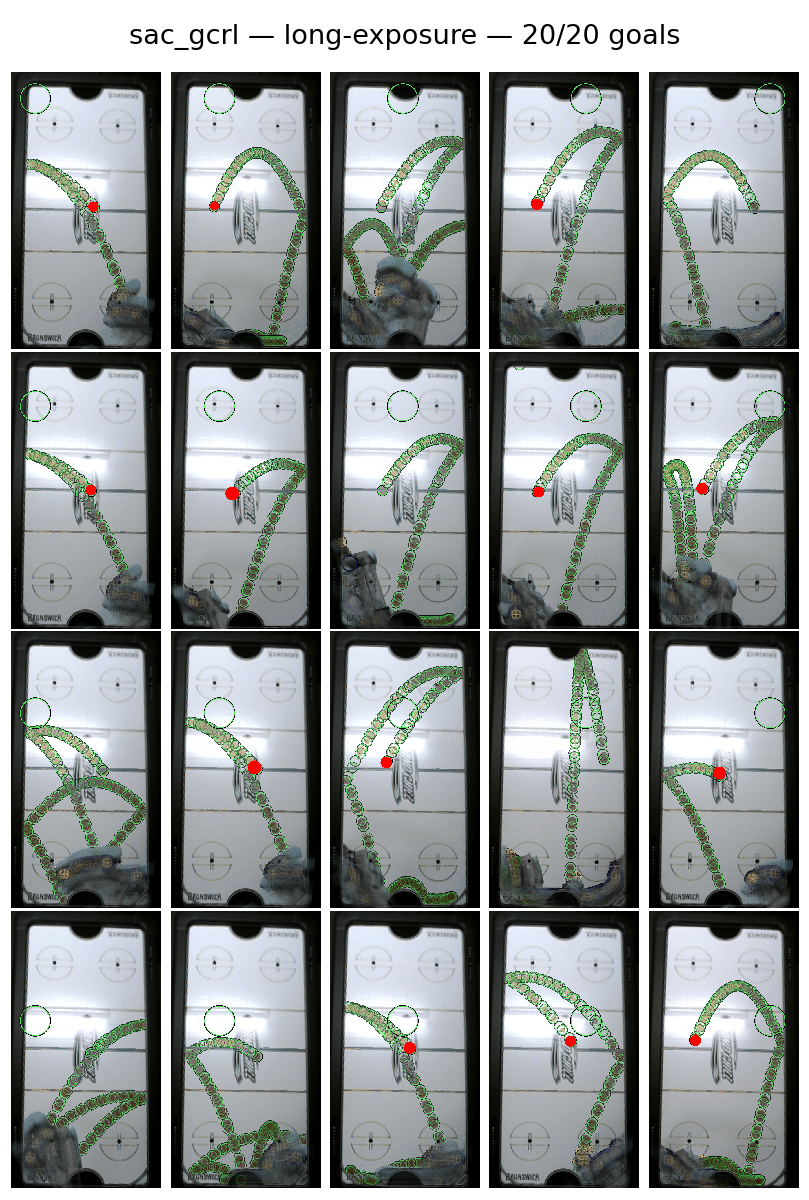}
        \caption{Long-exposure real-world rollout.}
        \label{fig:sac-gcrl-long-exposure}
    \end{subfigure}
    \caption{SAC-GCRL evaluation rollouts across 20/20 achieved goals.}
    \label{fig:sac-gcrl-paired}
\end{figure}

% -----------------------------
% SAC-HER
% -----------------------------
\begin{figure}[t]
    \centering
    \begin{subfigure}[t]{0.42\textwidth}
        \centering
        \includegraphics[width=\linewidth]{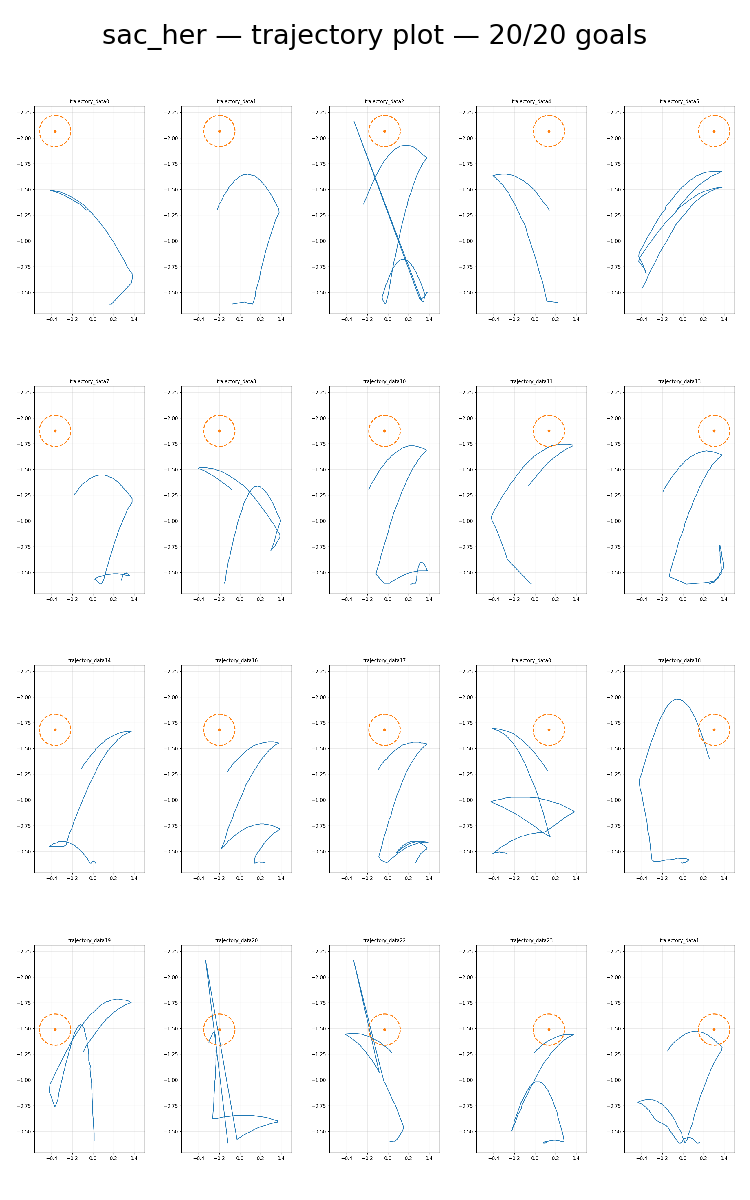}
        \caption{Trajectory-space visualization.}
        \label{fig:sac-her-trajectory}
    \end{subfigure}
    \hfill
    \begin{subfigure}[t]{0.42\textwidth}
        \centering
        \includegraphics[width=\linewidth]{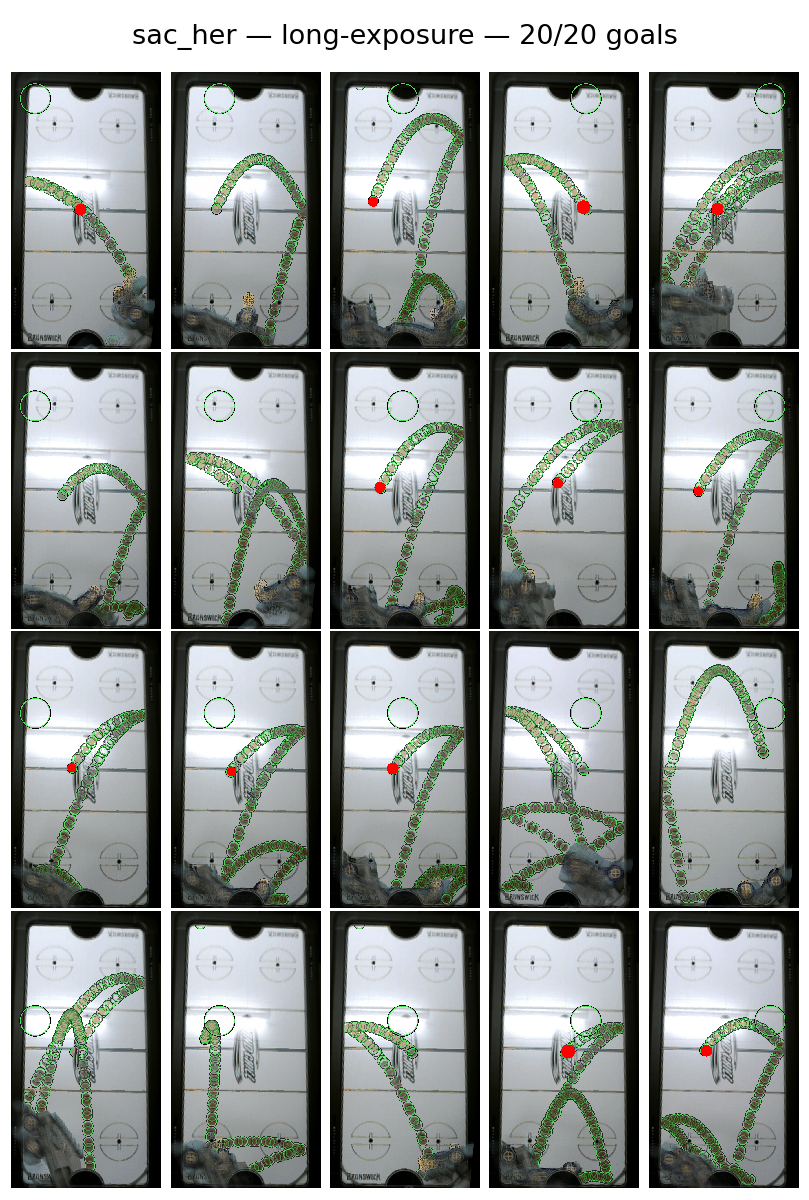}
        \caption{Long-exposure real-world rollout.}
        \label{fig:sac-her-long-exposure}
    \end{subfigure}
    \caption{SAC-HER evaluation rollouts across 20/20 achieved goals.}
    \label{fig:sac-her-paired}
\end{figure}

% -----------------------------
% SAC-Hint
% -----------------------------
\begin{figure}[t]
    \centering
    \begin{subfigure}[t]{0.42\textwidth}
        \centering
        \includegraphics[width=\linewidth]{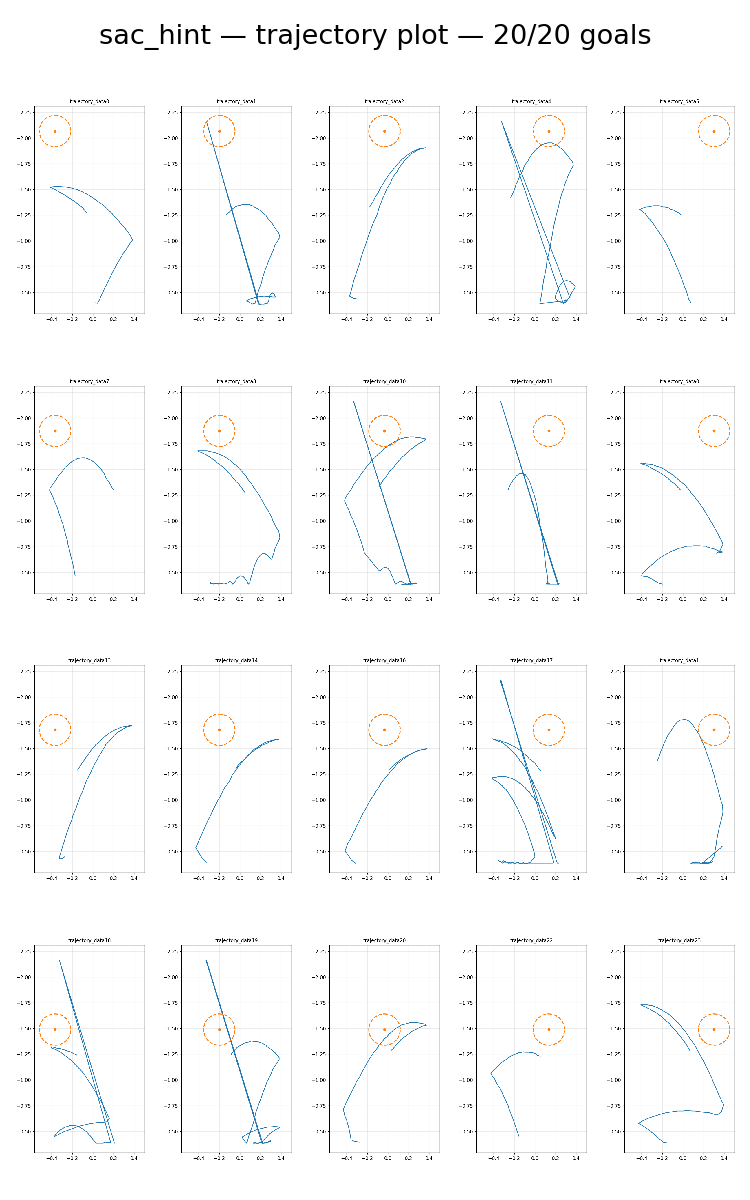}
        \caption{Trajectory-space visualization.}
        \label{fig:sac-hint-trajectory}
    \end{subfigure}
    \hfill
    \begin{subfigure}[t]{0.42\textwidth}
        \centering
        \includegraphics[width=\linewidth]{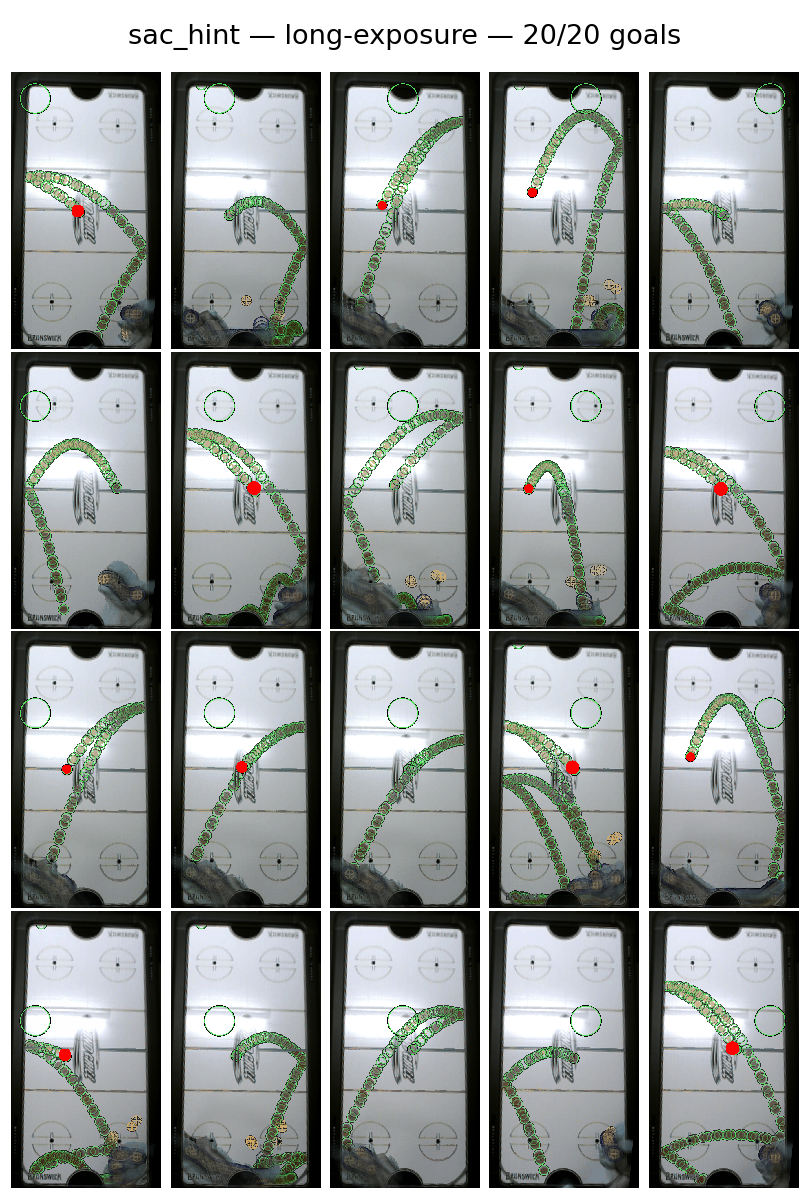}
        \caption{Long-exposure real-world rollout.}
        \label{fig:sac-hint-long-exposure}
    \end{subfigure}
    \caption{SAC-Hint evaluation rollouts across 20/20 achieved goals.}
    \label{fig:sac-hint-paired}
\end{figure}

% -----------------------------
% SGCRL Render
% -----------------------------
\begin{figure}[t]
    \centering
    \begin{subfigure}[t]{0.42\textwidth}
        \centering
        \includegraphics[width=\linewidth]{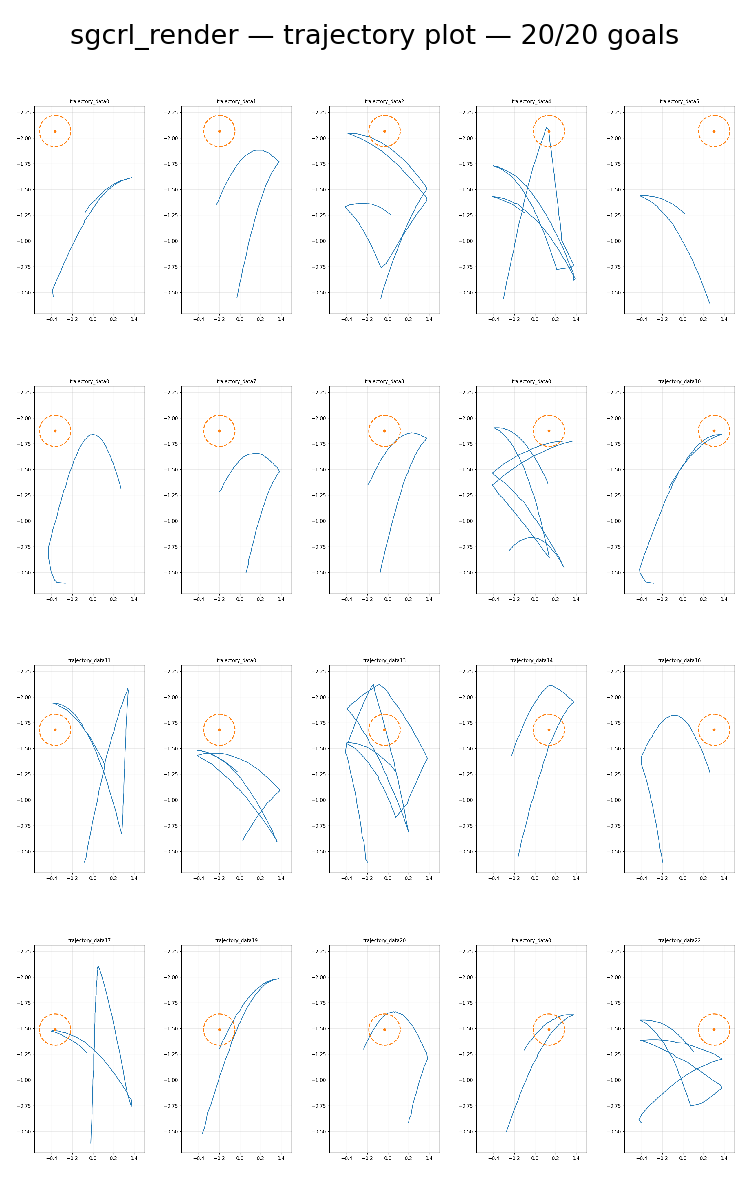}
        \caption{Trajectory-space visualization.}
        \label{fig:sgcrl-render-trajectory}
    \end{subfigure}
    \hfill
    \begin{subfigure}[t]{0.42\textwidth}
        \centering
        \includegraphics[width=\linewidth]{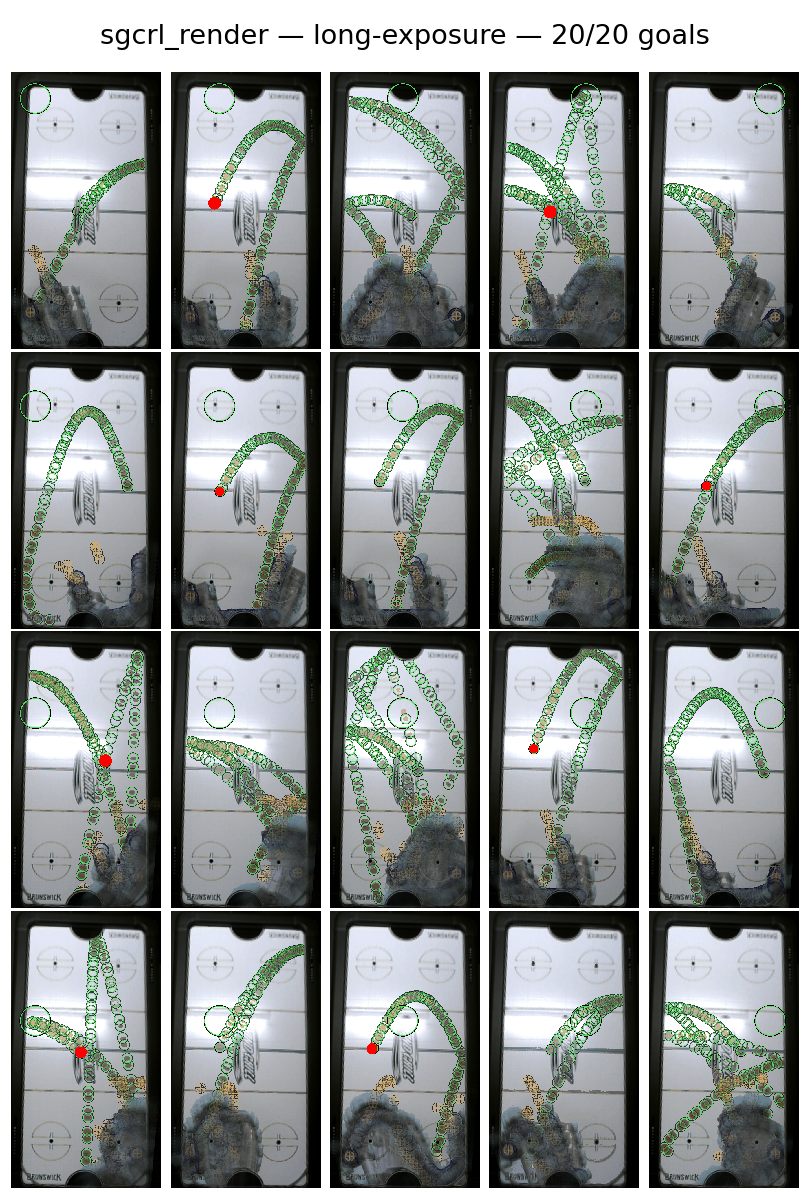}
        \caption{Long-exposure real-world rollout.}
        \label{fig:sgcrl-render-long-exposure}
    \end{subfigure}
    \caption{SGCRL evaluation rollouts across 20/20 achieved goals.}
    \label{fig:sgcrl-render-paired}
\end{figure}

\begin{figure}[t]
    \centering
    \begin{subfigure}[t]{0.42\textwidth}
        \centering
        \includegraphics[width=\linewidth]{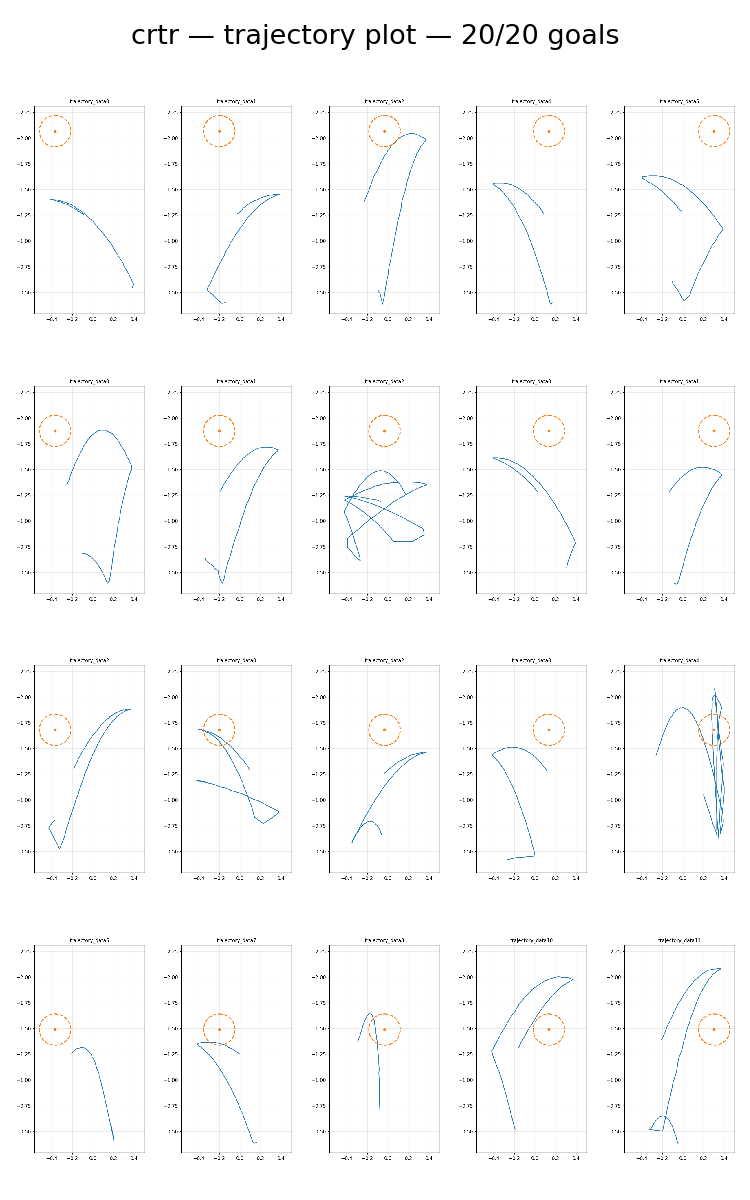}
        \caption{Trajectory-space visualization.}
        \label{fig:crtr-trajectory}
    \end{subfigure}
    \hfill
    \begin{subfigure}[t]{0.42\textwidth}
        \centering
        \includegraphics[width=\linewidth]{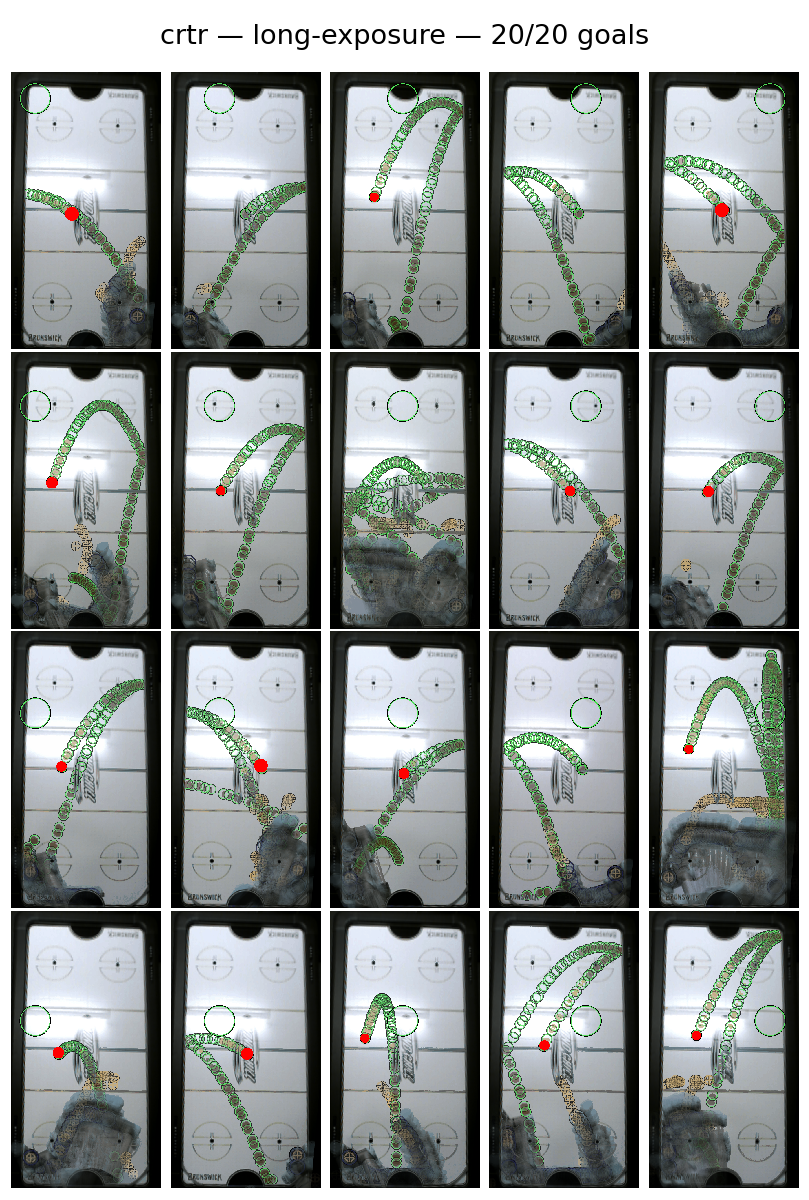}
        \caption{Long-exposure real-world rollout.}
        \label{fig:crtr-long-exposure}
    \end{subfigure}
    \caption{CRTR evaluation rollouts across 20/20 achieved goals.}
    \label{fig:crtr-paired}
\end{figure}

%% file: appendix/energy_visualization.tex
\section{Energy Visualization}
\label{app:energy_visualization}

Figure~\ref{fig:energy_landscape} shows the t-SNE visualization for learned $\phi(s,a)$. On locomotion tasks, e.g., Metaworld(Reach), the trajectory of latent space is smooth, while in manipulation tasks Air Hockey, the trajectory jumps whenever the interaction happens, which is indicated with a red circle. This result supports our analysis of mode-switching challenge of CRL.

\begin{figure}
    \centering
    \vspace{-1em}
    \includegraphics[width=0.98\textwidth]{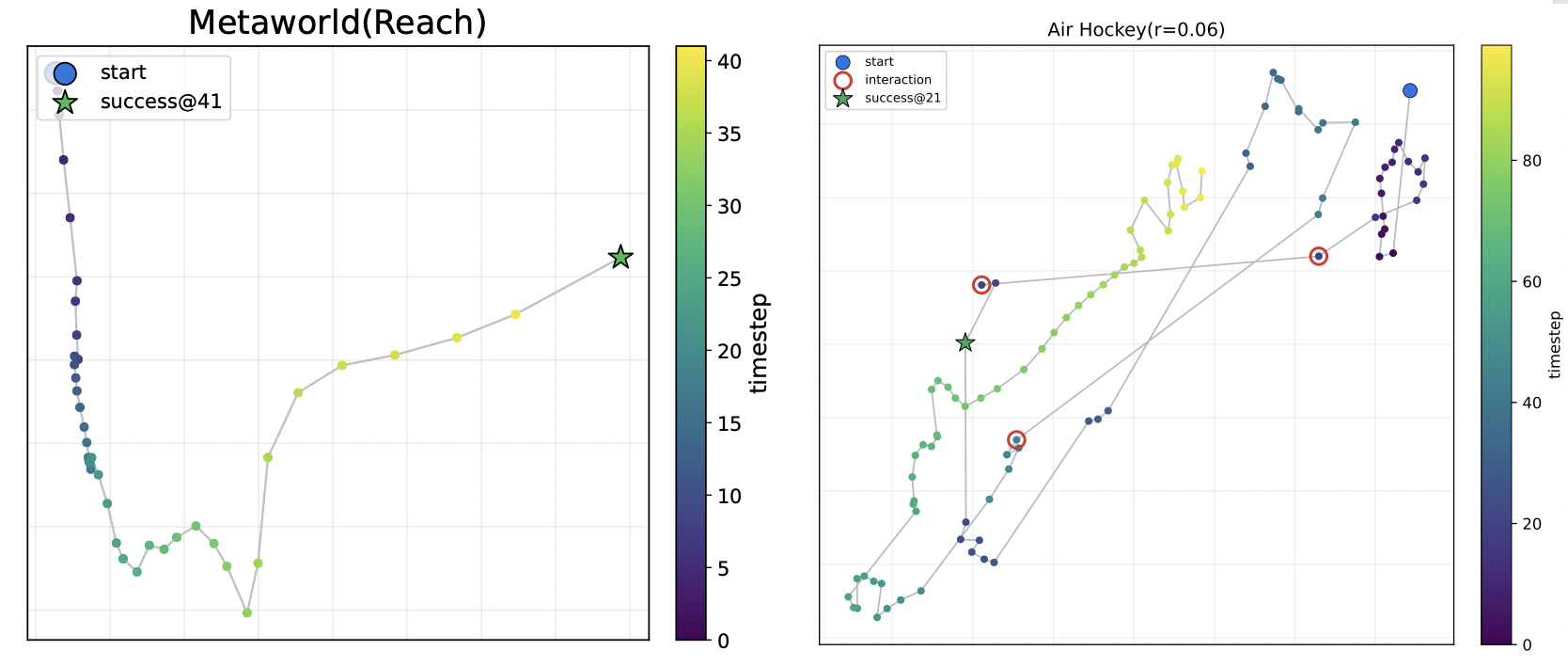}
    \caption{t-SNE visualization of $\phi(s,a)$ of SGCRL. Left: Metaworld(reach) Right: Air Hockey Simulator. Red circle indicates interactions. Star indicates success.
    }
    \label{fig:energy_landscape}
    \vspace{-1em}
\end{figure}

%% file: appendix/ablation.tex
\section{Detailed Experiment Result}
\label{app:ablation}

\subsection{Training curve and detailed table}

Figure~\ref{fig:app_training_curve} shows the training curve across 11 training tasks, and Table~\ref{app:tab_main_var} shows the standard deviation. Despite marginal improvement in tasks that has a high success rate to reach with Monte Carlo simulation, tasks like Air Hockey with a property of interaction dependency has a better performance. 

\begin{figure}[h]
    \vspace{-8pt}   
    \centering
    \small
\includegraphics[width=1.0\linewidth]{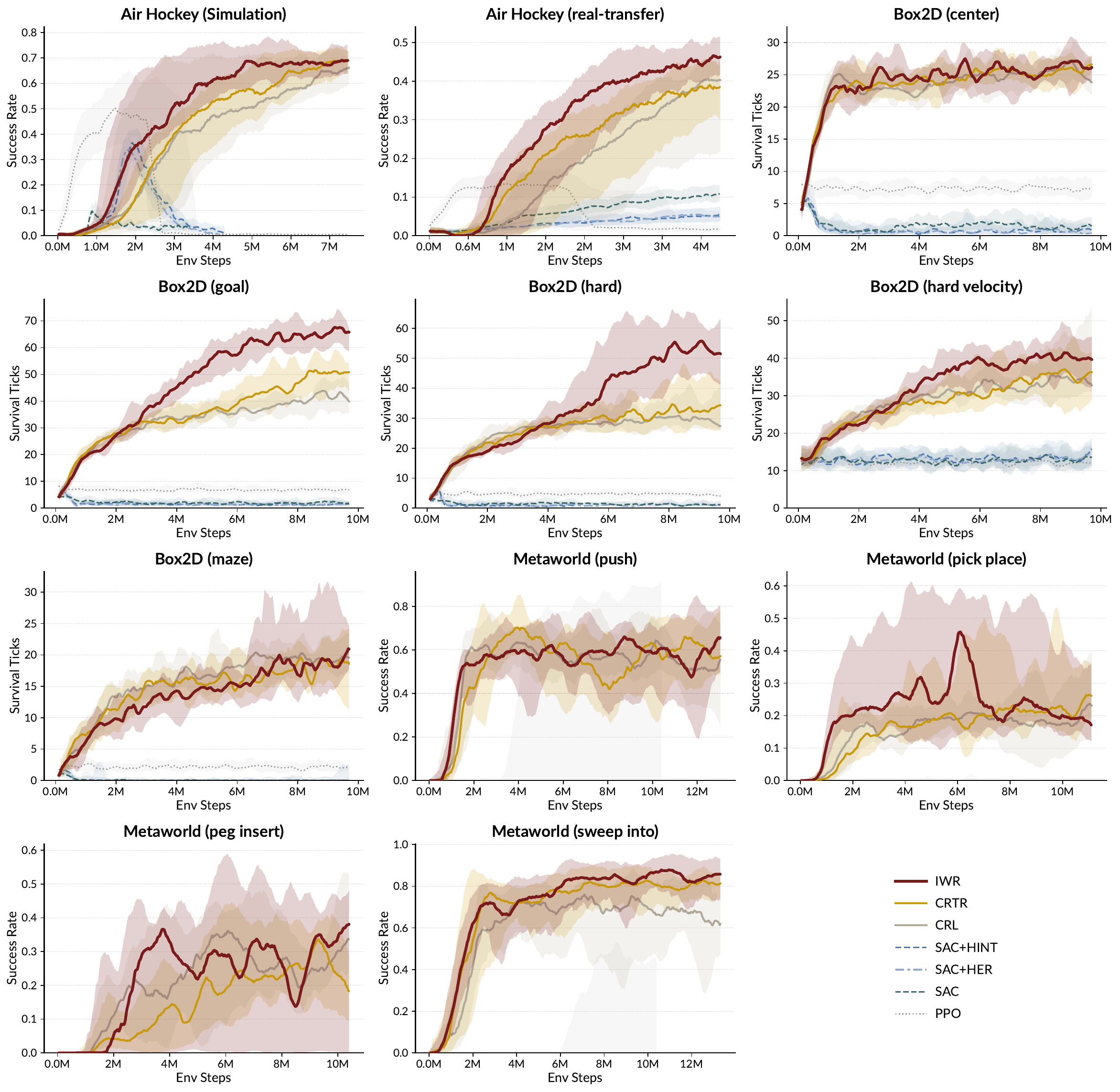}
    \caption{\textbf{Training curve for IWR compared to the baseline methods}}
    \label{fig:app_training_curve}
    \vspace{-.5cm}
\end{figure}

\input{table/table_main_var}

\subsection{Ablation}

We do ablation with a grid search on 6 tasks, with a multiplier on the hyperparameter, separately with 0.5 and 2.0 times of interaction threshold $c$ in Table~\ref{app:tab_ablation_c} and interaction weight $2\sigma^2$ in Table~\ref{app:tab_ablation_w}. Notice that the CRTR could be viewed as an ablation with $c \rightarrow \infty$, which is equivalently a uniform sampling. Notably all the methods has a better performance than the baseline, and the grid search result is even better than the selected result in the main paper. IWR is robust across a broad range of interaction-weight parameters, and further tuning can improve performance. This result also motivates that further works should focus on the line to improve interaction-related future sampling to better solve the problem of piecewise nonlinearity in the CRL latent space.

\input{table/table_ablation_c}
\input{table/table_ablation_weight}

\subsection{Real Robot Trajectory Visualization}

%% file: table/table_main_var.tex
% \documentclass{article}
% \usepackage[margin=0.55in]{geometry}
% \usepackage{booktabs}
% \usepackage{graphicx}
% \begin{document}
\begin{table*}[t]
\centering
\scriptsize
\setlength{\tabcolsep}{3pt}
\caption{Success rates over Air Hockey, Box2D, and MetaWorld tasks. Values report the central result from the main table with standard deviation across seeds.}
\label{tab:main1_std}
\resizebox{\textwidth}{!}{%
\begin{tabular}{lrrrrrrr}
\toprule
Task & PPO & SAC & SAC+HER & SAC+HINT & CRL & CRTR & IWR (Ours) \\
\midrule
Air Hockey (Simulation) & 0.617 $\pm$ 0.108 & 0.145 $\pm$ 0.029 & 0.398 $\pm$ 0.029 & 0.422 $\pm$ 0.032 & 0.695 $\pm$ 0.039 & 0.727 $\pm$ 0.017 & \textbf{0.742 $\pm$ 0.041 (+2.1\%)} \\
Air Hockey (real-transfer) & 0.160 $\pm$ 0.009 & 0.215 $\pm$ 0.018 & 0.129 $\pm$ 0.005 & 0.125 $\pm$ 0.006 & 0.477 $\pm$ 0.044 & 0.465 $\pm$ 0.014 & \textbf{0.500 $\pm$ 0.026 (+4.8\%)} \\
Air Hockey (Real Robot) & 0/20 & 0/20 & 0/20 & 0/20 & 5/20 & 2/20 & \textbf{12/20} \\
\midrule
Box2D (center) & 0.086 $\pm$ 0.006 & 0.058 $\pm$ 0.007 & 0.088 $\pm$ 0.015 & 0.088 $\pm$ 0.012 & 0.278 $\pm$ 0.013 & 0.274 $\pm$ 0.016 & \textbf{0.288 $\pm$ 0.011 (+3.6\%)} \\
Box2D (goal) & 0.089 $\pm$ 0.005 & 0.046 $\pm$ 0.004 & 0.086 $\pm$ 0.005 & 0.064 $\pm$ 0.011 & 0.450 $\pm$ 0.053 & 0.558 $\pm$ 0.051 & \textbf{0.709 $\pm$ 0.034 (+27.1\%)} \\
Box2D (hard) & 0.060 $\pm$ 0.006 & 0.042 $\pm$ 0.005 & 0.064 $\pm$ 0.009 & 0.076 $\pm$ 0.013 & 0.317 $\pm$ 0.050 & 0.365 $\pm$ 0.069 & \textbf{0.565 $\pm$ 0.084 (+54.8\%)} \\
Box2D (hard velocity) & 0.148 $\pm$ 0.010 & 0.149 $\pm$ 0.016 & 0.152 $\pm$ 0.007 & 0.139 $\pm$ 0.013 & 0.387 $\pm$ 0.075 & 0.377 $\pm$ 0.062 & \textbf{0.436 $\pm$ 0.036 (+12.7\%)} \\
Box2D (maze) & 0.033 $\pm$ 0.003 & 0.012 $\pm$ 0.003 & 0.031 $\pm$ 0.008 & 0.035 $\pm$ 0.018 & 0.217 $\pm$ 0.021 & 0.206 $\pm$ 0.024 & \textbf{0.223 $\pm$ 0.084 (+2.8\%)} \\
\midrule
MetaWorld (peg insert) & 0.000 $\pm$ 0.000 & 0.000 $\pm$ 0.000 & 0.000 $\pm$ 0.000 & 0.000 $\pm$ 0.000 & 0.430 $\pm$ 0.101 & 0.367 $\pm$ 0.103 & \textbf{0.438 $\pm$ 0.268 (+1.9\%)} \\
MetaWorld (pick place) & 0.000 $\pm$ 0.000 & 0.000 $\pm$ 0.000 & 0.004 $\pm$ 0.002 & 0.000 $\pm$ 0.000 & 0.266 $\pm$ 0.073 & 0.305 $\pm$ 0.163 & \textbf{0.570 $\pm$ 0.196 (+86.9\%)} \\
MetaWorld (push) & 0.000 $\pm$ 0.000 & 0.000 $\pm$ 0.000 & 0.004 $\pm$ 0.002 & 0.000 $\pm$ 0.000 & 0.699 $\pm$ 0.059 & \textbf{0.750 $\pm$ 0.064} & 0.730 $\pm$ 0.080 \\
MetaWorld (sweep into) & 0.000 $\pm$ 0.000 & 0.004 $\pm$ 0.002 & 0.020 $\pm$ 0.023 & 0.004 $\pm$ 0.005 & 0.805 $\pm$ 0.059 & 0.910 $\pm$ 0.017 & \textbf{0.926 $\pm$ 0.019 (+1.8\%)} \\
Average IWR improvement &  &  &  &  &  &  & \textbf{+19.8\%} \\
\bottomrule
\end{tabular}%
}
\label{app:tab_main_var}
\end{table*}
% \end{document}

%% file: table/table_ablation_c.tex
% \documentclass{article}
% \usepackage[margin=0.55in]{geometry}
% \usepackage{booktabs}
% \usepackage{graphicx}
% \begin{document}
\begin{table*}[t]
\centering
\small
\setlength{\tabcolsep}{4pt}
\caption{IWR threshold ablation with fixed $\sigma$ scale. Percentages compare against the better of CRL and CRTR.}
\label{tab:iwr_c_search}
\resizebox{\textwidth}{!}{%
\begin{tabular}{lrrrrr}
\toprule
Task & CRL & CRTR($c \rightarrow \infty$) & $c=0.5\times$ & $c=1.0\times$ & $c=1.5\times$ \\
\midrule
Air Hockey r0.06 & 0.695 $\pm$ 0.039 & 0.727 $\pm$ 0.017 & 0.715 $\pm$ 0.039 (-1.7\%) & 0.742 $\pm$ 0.041 (+2.1\%) & \textbf{0.770 $\pm$ 0.046 (+5.9\%)} \\
Air Hockey real-transfer & 0.477 $\pm$ 0.044 & 0.465 $\pm$ 0.014 & 0.496 $\pm$ 0.213 (+4.0\%) & 0.500 $\pm$ 0.026 (+4.8\%) & \textbf{0.543 $\pm$ 0.028 (+13.8\%)} \\
Box2D hard & 0.317 $\pm$ 0.050 & 0.365 $\pm$ 0.069 & \textbf{0.626 $\pm$ 0.020 (+71.4\%)} & 0.565 $\pm$ 0.084 (+54.8\%) & 0.532 $\pm$ 0.114 (+45.9\%) \\
Box2D hard velocity & 0.387 $\pm$ 0.075 & 0.377 $\pm$ 0.062 & \textbf{0.501 $\pm$ 0.031 (+29.5\%)} & 0.436 $\pm$ 0.036 (+12.7\%) & 0.462 $\pm$ 0.014 (+19.5\%) \\
MetaWorld pick-place & 0.266 $\pm$ 0.073 & 0.305 $\pm$ 0.163 & \textbf{0.801 $\pm$ 0.235 (+162.6\%)} & 0.570 $\pm$ 0.196 (+86.9\%) & 0.734 $\pm$ 0.304 (+140.8\%) \\
MetaWorld push & 0.699 $\pm$ 0.059 & 0.750 $\pm$ 0.064 & 0.867 $\pm$ 0.045 (+15.6\%) & 0.730 $\pm$ 0.080 (-2.7\%) & \textbf{0.887 $\pm$ 0.067 (+18.2\%)} \\
\midrule
Average improvement &  &  & \textbf{+46.9\%} & +26.4\% & +40.7\% \\
\bottomrule
\end{tabular}%
}
\label{app:tab_ablation_c}
\end{table*}
% \end{document}

%% file: table/table_ablation_weight.tex
% \documentclass{article}
% \usepackage[margin=0.55in]{geometry}
% \usepackage{booktabs}
% \usepackage{graphicx}
% \begin{document}
\begin{table*}[t]
\centering
\small
\setlength{\tabcolsep}{4pt}
\caption{IWR $2\sigma^2$ ablation with fixed contact threshold scale. Percentages compare against the better of CRL and CRTR.}
\label{tab:iwr_sigma_search}
\resizebox{\textwidth}{!}{%
\begin{tabular}{lrrrrr}
\toprule
Task & CRL & CRTR & $2\sigma^2=0.5 \times$ & $2\sigma^2=1.0\times$ & $2\sigma^2=2.0\times$ \\
\midrule
Air Hockey r0.06 & 0.695 $\pm$ 0.039 & 0.727 $\pm$ 0.017 & 0.789 $\pm$ 0.040 (+8.5\%) & 0.742 $\pm$ 0.041 (+2.1\%) & \textbf{0.816 $\pm$ 0.042 (+12.3\%)} \\
Air Hockey real-transfer & 0.477 $\pm$ 0.044 & 0.465 $\pm$ 0.014 & \textbf{0.562 $\pm$ 0.030 (+17.9\%)} & 0.500 $\pm$ 0.026 (+4.8\%) & 0.516 $\pm$ 0.035 (+8.1\%) \\
Box2D hard & 0.317 $\pm$ 0.050 & 0.365 $\pm$ 0.069 & 0.652 $\pm$ 0.060 (+78.6\%) & 0.565 $\pm$ 0.084 (+54.8\%) & \textbf{0.677 $\pm$ 0.094 (+85.4\%)} \\
Box2D hard velocity & 0.387 $\pm$ 0.075 & 0.377 $\pm$ 0.062 & 0.459 $\pm$ 0.035 (+18.5\%) & 0.436 $\pm$ 0.036 (+12.7\%) & \textbf{0.496 $\pm$ 0.040 (+28.3\%)} \\
MetaWorld pick-place & 0.266 $\pm$ 0.073 & 0.305 $\pm$ 0.163 & 0.844 $\pm$ 0.266 (+176.6\%) & 0.570 $\pm$ 0.196 (+86.9\%) & \textbf{0.848 $\pm$ 0.279 (+177.9\%)} \\
MetaWorld push & 0.699 $\pm$ 0.059 & 0.750 $\pm$ 0.064 & \textbf{0.816 $\pm$ 0.059 (+8.9\%)} & 0.730 $\pm$ 0.080 (-2.7\%) & 0.801 $\pm$ 0.032 (+6.8\%) \\
\midrule
Average improvement &  &  & +51.5\% & +26.4\% & \textbf{+53.1\%} \\
\bottomrule
\end{tabular}%
}
\end{table*}
\label{app:tab_ablation_w}
% \end{document}